%% file: main.tex
\documentclass[10pt,twocolumn,letterpaper]{article}

\pdfoutput=1

\usepackage[pagenumbers]{cvpr}

\input{preamble/preamble.tex}
\input{preamble/title.tex}

\input{tab/methods.tex}

\input{tab/hyperparameters.tex}

\begin{document}
    \title{\theTitle}
    \author{\theAuthor}
    \maketitle

    \begin{abstract}
        \input{inputs/abstract.tex}
    \end{abstract}

    \section{Introduction}
        \input{inputs/intro.tex}

    \section{\fwName framework}
    \label{sec:framework}
        \input{inputs/framework.tex}

    \section{Previous methods as instances of \fwName}
    \label{sec:relatedWork}
        \input{inputs/related-work.tex}

    \section{Contrastive alignment in deep MVC}
    \label{sec:motivatingExperiments}
        \input{inputs/motivating-experiments.tex}

    \section{New instances of \fwName}
    \label{sec:newVariations}
        \input{inputs/new-instances.tex}

    \section{Experiments}
    \label{sec:experiments}
        \input{inputs/experiments.tex}

    \section{Conclusion}
    \label{sec:conclusion}
        \input{inputs/conclusion.tex}

    \section*{Acknowledgements}
        This work was financially supported by the Research Council of Norway (RCN), through its Centre for Research-based Innovation funding scheme (Visual Intelligence, grant no.\ 309439), and Consortium Partners.
        It was further funded by RCN FRIPRO grant no.\ 315029, RCN IKTPLUSS grant no.\ 303514, and the UiT Thematic Initiative ``Data-Driven Health Technology''.

    \appendix
    \vspace{1cm}
    {\huge\bfseries\noindent Supplementary material}

    \section{Introduction}
        \input{supp-inputs/intro.tex}

    \section{Previous methods as instances of \fwName}
        \input{supp-inputs/related-work.tex}

    \section{Contrastive alignment in deep MVC}
        \input{supp-inputs/alignment-deep-mvc.tex}

    \section{New instances of \fwName}
        \input{supp-inputs/new-instances.tex}

    \section{Experiments}
        \input{supp-inputs/experiments.tex}

    \section{Potential negative societal impacts}
        \input{supp-inputs/negative-societal-impact.tex}

    \bibliographystyle{unsrtnat}
    \bibliography{references}

\end{document}

%% file: preamble/preamble.tex
\usepackage[utf8]{inputenc}
\usepackage[T1]{fontenc}
\usepackage{amsfonts}
\usepackage{nicefrac}
\usepackage[table]{xcolor}
\usepackage{graphicx}
\usepackage{amsmath}
\usepackage{amssymb}
\usepackage{booktabs}
\usepackage[english]{babel}
\usepackage{amsthm}
\usepackage{microtype}
\usepackage[inline]{enumitem}
\usepackage{dsfont}
\usepackage{multirow}
\usepackage{pifont}
\usepackage{lscape}
\usepackage{tabularx}
\usepackage{fp}
\usepackage[super]{nth}
\usepackage{xstring}
\usepackage{floatrow}
\usepackage{calc}
\usepackage{wrapfig}
\usepackage{floatflt}
\usepackage{mathtools}
\usepackage[numbers,square,sort&compress]{natbib}
\usepackage[pagebackref,breaklinks,colorlinks]{hyperref}

\makeatletter
\@namedef{ver@everyshi.sty}{}
\makeatother
\usepackage{tikz}

\let\vec\boldsymbol
\newcommand{\lrp}[1]{\left( #1 \right)}

\newcommand{\lrc}[1]{\left\{ #1 \right\}}

\newcommand{\sums}[2]{\sum\limits_{#1}^{#2}}
\newcommand{\cl}[1]{\mathcal{#1}}
\newcommand{\T}{^{\top}}
\newcommand{\real}{\mathbb{R}}
\newcommand{\naturals}{\mathbb{N}}

\newcommand{\E}{\mathbb{E}}
\newcommand{\Prob}{\mathbb{P}}

\newcommand{\customparagraph}[1]{\vspace{-.1em}\noindent\textbf{#1}~}

\let\exp\undefined
\DeclareMathOperator{\exp}{exp}

\makeatletter
\DeclareRobustCommand\onedot{\futurelet\@let@token\@onedot}
\def\@onedot{\ifx\@let@token.\else.\null\fi\xspace}

\def\eg{\emph{e.g}\onedot} 
\def\ie{\emph{i.e}\onedot} 
 
\def\etc{\emph{etc}\onedot} 
 
\def\etal{\emph{et al}\onedot}
\makeatother

\newfloatcommand{capbtabbox}{table}[][\FBwidth]

\newcommand{\tableFontSize}{\small}
\newcommand{\undoCaptionSep}[1][-.3cm]{\vspace{#1}}

\newcommand{\ST}[2][l]{%
    \begin{tabular}{@{}#1@{}}#2\end{tabular}%
}

\definecolor{posDelta}{rgb}{0.0, 0.8, 0.0}
\definecolor{negDelta}{rgb}{1.0, 0.0, 0.0}

\newcommand{\MTC}[1]{\( #1 \)}
\newcommand{\BEST}[1]{\mathbf{#1}}
\newcommand{\STD}[1]{{\scriptsize (\( #1 \))}}
\newcommand{\DELTA}[1]{%
    \(
    \IfDecimal{#1}{%
        \FPifneg{#1}%
            \textcolor{negDelta}{#1}%
        \else\fi%
        \FPifgt{#1}{0}%
            \textcolor{posDelta}{+#1}%
        \else\fi%
        \FPifzero{#1}%
            #1%
        \else\fi%
    }
    {#1}
    \)%
}
\newcommand{\MTCDELTA}[2]{\MTC{#1}~{\scriptsize (\DELTA{#2})}}

\newcommand{\TRUE}{\ding{51}}
\newcommand{\FALSE}{\ding{56}}

\newcommand{\fwName}{DeepMVC\xspace}

\newcommand{\dmsc}{DMSC}
\newcommand{\mvscn}{MvSCN}
\newcommand{\eamc}{EAMC}
\newcommand{\simvc}{SiMVC}
\newcommand{\comvc}{CoMVC}

\newcommand{\mvae}{Multi-VAE}

\newcommand{\MethodName}[3]{#1#2#3}
\newcommand{\sae}{\MethodName{AE}{--}{DDC}\xspace}
\newcommand{\cae}{\MethodName{AE}{Co}{DDC}\xspace}
\newcommand{\saekm}{\MethodName{AE}{--}{KM}\xspace}
\newcommand{\caekm}{\MethodName{AE}{Co}{KM}\xspace}
\newcommand{\mimvc}{InfoDDC\xspace}
\newcommand{\mviic}{MV-IIC\xspace}

\newcommand{\TOTLoss}{\cl L^{\text{Total}}}
\newcommand{\SVSSLLoss}{\cl L^{\text{SV}}}
\newcommand{\MVSSLLoss}{\cl L^{\text{MV}}}
\newcommand{\CMLoss}{\cl L^{\text{CM}}}

\newcommand{\SVSSLWeight}{w^{\text{SV}}}
\newcommand{\MVSSLWeight}{w^{\text{MV}}}
\newcommand{\CMWeight}{w^{\text{CM}}}

\usepackage{pgfplots}
\DeclareUnicodeCharacter{2212}{−}
\usepgfplotslibrary{groupplots,dateplot}
\usetikzlibrary{patterns,shapes.arrows}
\pgfplotsset{compat=newest}

\definecolor{darkblue}{rgb}{0, 0, .6}
\newcommand{\suppdir}[1]{\textcolor{darkblue}{{\small\texttt{#1}}}}
\newcommand{\cthead}[1]{\multicolumn{1}{c}{#1}}

\newtheorem{proposition}{Proposition}

%% file: preamble/title.tex
\def\theTitle{On the Effects of Self-supervision and Contrastive Alignment in Deep Multi-view Clustering}

\def\asp{~~}
\def\tsp{\hspace{.8ex}}
\def\theAuthor{%
    Daniel J. Trosten%
        \thanks{UiT Machine Learning group ({\url{machine-learning.uit.no}}) and Visual Intelligence Centre (\url{visual-intelligence.no}).},%
    \asp Sigurd Løkse%
        \footnotemark[1],%
    \asp Robert Jenssen%
        \footnotemark[1]\tsp%
        \thanks{Norwegian Computing Center (\url{nr.no}).}\tsp%
        \thanks{Department of Computer Science, University of Copenhagen.}\tsp%
        \thanks{Pioneer Centre for AI (\url{aicentre.dk}).},%
    \asp Michael C. Kampffmeyer%
        \footnotemark[1]\tsp%
        \footnotemark[2]\\
    Department of Physics and Technology,%
    ~UiT The Arctic University of Norway\\
    {\tt \small firstname[.middle initial].lastname@uit.no}
}

\newcommand{\githubLink}[1][\small]{{#1 \url{https://github.com/DanielTrosten/DeepMVC}}}

%% file: tab/methods.tex
\def\NOTSPEC{?\xspace}
\def\NOTSPECLong{Not specified}

\def\NOTINC{--\xspace}
\def\NOTINCLong{Not included}

\def\AL{Al.}
\def\ALLong{Alignment}

\def\MLP{MLP\xspace}

\def\CNN{CNN\xspace}

\def\REC{Reconstruction\xspace}

\def\SE{Sp.~Emb.\xspace}
\def\SELong{Spectral Embedding}

\def\VREC{Variational \REC\xspace}

\def\MNSI{\ST{Min.~superflous\\information}\xspace}

\def\NGH{Ngh.~preserv.\xspace}
\def\NGHLong{Neighborhood preservation}

\def\CCA{CCA\xspace}
\def\CCALong{Canonical correlation analysis}

\def\MSEA{MSE \AL\xspace}

\def\CU{\ST{Commonness\\uniqueness}\xspace}

\def\DA{Distribution \AL\xspace}

\def\KA{Kernel \AL\xspace}

\def\CA{Contrastive \AL\xspace}

\def\MXSI{\ST{Max.~shared\\information}\xspace}

\def\MXMI{Max.~mutual info.\xspace}

\def\OC{IIC Overclustering}

\def\IB{Inf.~Bottleneck\xspace}
\def\IBLong{Information bottleneck}

\def\FV{\nth{1}~view\xspace}

\def\AF{\ST{Affinity\\fusion}\xspace}

\def\WS{Weighted sum\xspace}

\def\CAT{Concat.\xspace}
\def\CATLong{Concatenate}

\def\SHNET{\ST{Shared\\network}\xspace}

\def\AVG{Average\xspace}

\def\ATT{Attention\xspace}

\def\SC{SC\xspace}
\def\SCLong{Spectral clustering}

\def\SR{SR\xspace}
\def\SRLong{Self-representation}

\def\SSR{Sparse SR\xspace}
\def\SSRLong{Sparse self-representation}

\def\KM{\( k \)-means\xspace}

\def\FO{\ST{Fusion\\output}\xspace}

\def\GMM{GMM\xspace}

\def\DEC{DEC\xspace}
\def\DECLong{Deep embedded clustering}

\def\DDC{DDC\xspace}
\def\DDCLong{Deep divergence-based clustering}

\def\GS{Gumbel\xspace}

\def\EO{\ST{Encoder\\output}\xspace}

\def\IIC{IIC}

\def\PreviousMethodsTable{
    \setlength{\tabcolsep}{1mm}
    \rowcolors{2}{gray!25}{white}
    \setlength{\lightrulewidth}{0.1em}
    \begin{tabular}{lllllll} \toprule \rowcolor{gray!40}
        \textbf{Model}                                   & \textbf{Pub.} & \textbf{Enc.}      & \textbf{SV-SSL}  & \textbf{MV-SSL} & \textbf{Fusion} & \textbf{CM} \\ %
        DCCAE~\cite{wangDeepMultiViewRepresentation2015}                & ICML'15          & \MLP              & \REC             & \CCA            & \FV             & \SC \\
        DMSC~\cite{abavisaniDeepMultimodalSubspace2018}                 & J.~STSP'18       & \CNN              & \REC             & \NOTINC         & \AF             & \SR, \SC \\
        DMVSSC~\cite{tangDeepMultiviewSparse2018}                       & ICNCC'18         & \CNN              & \REC             & \NOTINC         & \NOTINC         & \ST{\SSR,\\\SC} \\
        MvSN~\cite{huangMultiSpectralNetSpectralClustering2019}         & T.~CSS'19        & \MLP              & \SE              & \NOTINC         & \WS             & \KM \\
        MvSCN~\cite{huangMultiviewSpectralClustering2019}               & IJCAI'19         & \MLP              & \SE              & \MSEA           & \CAT            & \KM \\
        MvDSCN~\cite{zhuMultiviewDeepSubspace2019}                      & arXiv'19         & \CNN              & \NOTINC          & \REC            & \SHNET          & \SR, \SC \\
        DAMC~\cite{liDeepAdversarialMultiview2019}                      & IJCAI'19         & \MLP              & \NOTINC          & \REC            & \AVG            & \DEC \\
        S2DMVSC~\cite{sunSelfSupervisedDeepMultiView2019}               & ACML'19          & \MLP              & \REC             & \NOTINC         & \MLP            & \SR, \SC \\
        DCMR~\cite{zhangDeepMultimodalClustering2020}                   & PAKDD'20         & \MLP              & \VREC            & \VREC           & \MLP            & \KM \\
        DMMC~\cite{zhangEndToEndDeepMultimodal2020}                     & ICME'20          & \MLP              & \REC             & \NOTINC         & \MLP            & \FO \\
        DCUMC~\cite{zongMultimodalClusteringDeep2020}                   & ICIKM'20         & \MLP              & \REC             & \CU             & \MLP            & \KM \\
        SGLR-MVC~\cite{yinSharedGenerativeLatent2020}                   & AAAI'20          & \MLP              & \NOTINC          & \VREC           & \WS             & \GMM \\
        EAMC~\cite{zhouEndtoEndAdversarialAttentionNetwork2020}         & CVPR'20          & \MLP              & \NOTINC          & \ST{\DA,\\ \KA} & \ATT            & \DDC \\
        MVC-MAE~\cite{duDeepMultipleAutoEncoderBased2021}               & DSE'21           & \MLP              & \ST{\REC,\\\NGH} & \CA             & \NOTINC         & \DEC \\
        SDC-MVC~\cite{xinSelfSupervisedDeepCorrelational2021}           & IJCNN'21         & \MLP              & \NOTINC          & \CCA            & \CAT            & \DEC \\
        DEMVC~\cite{xuDeepEmbeddedMultiview2021}                        & Inf.~Sci.'21     & \CNN              & \REC             & \NOTINC         & \NOTINC         & \DEC \\
        SiMVC~\cite{trostenReconsideringRepresentationAlignment2021}    & CVPR'21          & \ST{\MLP/\\ \CNN} & \NOTINC          & \NOTINC         & \WS             & \DDC \\
        CoMVC~\cite{trostenReconsideringRepresentationAlignment2021}    & CVPR'21          & \ST{\MLP/\\ \CNN} & \NOTINC          & \CA             & \WS             & \DDC \\
        Multi-VAE~\cite{xuMultiVAELearningDisentangled2021}             & ICCV'21          & \CNN              & \NOTINC          & \VREC           & \CAT            & \ST{\GS,\\\KM} \\
        DMIM~\cite{maoDeepMutualInformation2021}                        & IJCAI'21         & \MLP              & \MNSI            & \MXSI           & \NOTSPEC        & \EO \\
        AMvC~\cite{wangAdversarialMultiviewClustering2022}              & TNNLS'22         & \MLP              & \NOTINC          & \REC            & \WS             & \DEC \\
        SIB-MSC~\cite{wangSelfSupervisedInformationBottleneck2022}      & arXiv'22         & \CNN              & \NOTINC          & \ST{\REC,\\\IB} & \AF             & \SR, \SC \\
    \bottomrule
    \end{tabular}
}

\def\MethodsAbbreviations{
    \begin{flushleft}
        \textbf{Abbreviations:}
        ``\NOTINC'' = \NOTINCLong,
        ``\NOTSPEC'' = \NOTSPECLong,
        \AL~= \ALLong,
        \CAT = \CATLong,
        \CCA = \CCALong,
        \DDC = \DDCLong,
        \DEC = \DECLong,
        \IB = \IBLong,
        \NGH = \NGHLong,
        \SC = \SCLong,
        \SE = \SELong,
        \SR = \SRLong,
        \SSR = \SSRLong,
    \end{flushleft}
}

\def\PreviousAndNewMethodsTable{
    \setlength{\tabcolsep}{.9mm}
    \rowcolors{2}{gray!25}{white}
    \setlength{\lightrulewidth}{0.1em}
    \begin{tabular}{lllllll} \toprule \rowcolor{gray!40}
        \textbf{Model}                                                            & \textbf{Pub.} & \textbf{Enc.} & \textbf{SV-SSL}  & \textbf{MV-SSL} & \textbf{Fusion} & \textbf{CM} \\ %
        DCCAE~\cite{wangDeepMultiViewRepresentation2015}                          & ICML'15        & \MLP              & \REC             & \CCA            & \FV             & \SC \\
        DMSC~\cite{abavisaniDeepMultimodalSubspace2018}                           & J.~STSP'18     & \CNN              & \REC             & \NOTINC         & \AF             & \SR, \SC \\
        MvSCN~\cite{huangMultiviewSpectralClustering2019}                         & IJCAI'19       & \MLP              & \SE              & \MSEA           & \CAT            & \KM \\
        DAMC~\cite{liDeepAdversarialMultiview2019}                                & IJCAI'19       & \MLP              & \NOTINC          & \REC            & \AVG            & \DEC \\
        SGLR-MVC~\cite{yinSharedGenerativeLatent2020}                             & AAAI'20        & \MLP              & \VREC            & \VREC           & \WS             & \GMM \\
        EAMC~\cite{zhouEndtoEndAdversarialAttentionNetwork2020}                   & CVPR'20        & \MLP              & \NOTINC          & \ST{\DA,\\ \KA} & \ATT            & \DDC \\
        SiMVC~\cite{trostenReconsideringRepresentationAlignment2021}              & CVPR'21        & \MLP/\CNN         & \NOTINC          & \NOTINC         & \WS             & \DDC \\
        CoMVC~\cite{trostenReconsideringRepresentationAlignment2021}              & CVPR'21        & \MLP/\CNN         & \NOTINC          & \CA             & \WS             & \DDC \\
        Multi-VAE~\cite{xuMultiVAELearningDisentangled2021}                       & ICCV'21        & \CNN              & \NOTINC          & \VREC           & \CAT            & \ST{\GS,\\\KM} \\
        DMIM~\cite{maoDeepMutualInformation2021}                                  & IJCAI'21       & \MLP              & \MNSI            & \MXSI           & \NOTSPEC        & \EO \\
        \midrule \rowcolor{gray!40}
        \textbf{Model} & \textbf{Category}        & \textbf{Enc.}      & \textbf{SV-SSL}            & \textbf{MV-SSL} & \textbf{Fusion} & \textbf{CM} \\
        \saekm         & Simple                   & \MLP/\CNN   & \REC                       & \NOTINC         & \CAT            & \KM \\
        \sae           & Simple                   & \MLP/\CNN   & \REC                       & \NOTINC         & \WS             & \DDC \\
        \caekm         & \CA                      & \MLP/\CNN   & \REC                       & \CA             & \CAT            & \KM \\
        \cae           & \CA                      & \MLP/\CNN   & \REC                       & \CA             & \WS             & \DDC \\
        \mimvc         & Mutual info.             & \MLP/\CNN   & \NOTINC                    & \MXMI           & \WS             & \DDC \\
        \mviic         & Mutual info.             & \MLP/\CNN   & \NOTINC                    & \OC             & \NOTINC         & \IIC, \KM \\
    \bottomrule
    \end{tabular}
}

%% file: tab/hyperparameters.tex
\def\layer#1{\texttt{#1}}
\def\Conv{\layer{Conv}}
\def\BatchNormalization{\layer{BatchNorm}}
\def\RELU{\layer{ReLU}}
\def\MaxPool{\layer{MaxPool}}
\def\Dense{\layer{Dense}}
\def\ConvTranspose{\layer{TransposeConv}}
\def\UpSample{\layer{UpSample}}
\def\Sigmoid{\layer{Sigmoid}}

\def\AllArch{
    \begin{tabular}{llll} \toprule
        \cthead{CNN encoder} & \cthead{CNN decoder} & \cthead{MLP encoder} & \cthead{MLP decoder} \\ \midrule
        \Conv \( (64 \times 3 \times 3) \) & \UpSample \( (2 \times 2) \) & \Dense \( (1024) \) & \Dense \( (256) \) \\
        \RELU & \ConvTranspose \( (64 \times 3 \times 3) \) & \BatchNormalization & \BatchNormalization \\
        \Conv \( (64 \times 3 \times 3) \) & \RELU & \RELU & \RELU \\
        \BatchNormalization & \ConvTranspose \( (64 \times 3 \times 3) \) & \Dense \( (1024) \) & \Dense \( (1024) \) \\
        \RELU & \BatchNormalization & \BatchNormalization & \BatchNormalization \\
        \MaxPool \( (2 \times 2) \) & \RELU & \RELU & \RELU \\
        \Conv \( (64 \times 3 \times 3) \) & \UpSample \( (2 \times 2) \) & \Dense \( (1024) \) & \Dense \( (1024) \) \\
        \RELU & \ConvTranspose \( (64 \times 3 \times 3) \) & \BatchNormalization & \BatchNormalization \\
        \Conv \( (64 \times 3 \times 3) \) & \RELU & \RELU & \RELU \\
        \BatchNormalization & \ConvTranspose \( (1 \times 3 \times 3) \) & \Dense \( (1024) \) & \Dense \( (1024) \) \\
        \RELU & \Sigmoid & \BatchNormalization & \BatchNormalization \\
        \MaxPool \( (2 \times 2) \) & & \RELU & \RELU \\
        && \Dense \( (256) \) & \Dense \( (\texttt{input dim}) \) \\
        &&& \Sigmoid \\
        \bottomrule
    \end{tabular}
}

\def\EAMCLearningRate{\( \dagger \) = EAMC~\cite{zhouEndtoEndAdversarialAttentionNetwork2020} has different learning rates for the different components, namely \( 10^{-5} \) for the encoders and clustering module, and \( 10^{-4} \) for the attention module and discriminator.}

\def\hyperparametersAll{
    \begin{tabular}{lccccccc} \toprule
        Model  & Batch size & Learning rate         & \( \SVSSLWeight \) & \( \MVSSLWeight \) & \( \CMWeight \) & Pre-train & Gradient clip \\ \cmidrule(lr){1-1} \cmidrule(lr){2-8}
        \dmsc  & \( 100 \)  & \( 10^{-3} \)         & \( 1.0 \)          & --                 & --              & \TRUE     & \( 10 \) \\
        \mvscn & \( 512 \)  & \( 10^{-4} \)         & \( 0.999 \)        & \( 0.001 \)        & --              & \FALSE    & \( 10 \) \\
        \eamc  & \( 100 \)  & \( \dagger \)         & --                 & \( 1.0 \)          & \( 1.0 \)       & \FALSE    & \( 10 \) \\
        \simvc & \( 100 \)  & \( 10^{-3} \)         & --                 & --                 & \( 1.0 \)       & \FALSE    & \( 10 \) \\
        \comvc & \( 100 \)  & \( 10^{-3} \)         & --                 & \( 0.1 \)          & \( 1.0 \)       & \FALSE    & \( 10 \) \\
        \mvae  & \( 64 \)   & \( 5 \cdot 10^{-4} \) & --                 & \( 1.0 \)          & --              & \TRUE     & \( 10 \) \\
        \sae   & \( 100 \)  & \( 10^{-3} \)         & \( 1.0 \)          & --                 & \( 1.0 \)       & \TRUE     & \( 10 \) \\
        \cae   & \( 100 \)  & \( 10^{-3} \)         & \( 1.0 \)          & \( 0.1 \)          & \( 1.0 \)       & \TRUE     & \( 10 \) \\
        \saekm & \( 100 \)  & \( 10^{-3} \)         & \( 1.0 \)          & --                 & --              & \FALSE    & \( 10 \) \\
        \caekm & \( 100 \)  & \( 10^{-3} \)         & \( 1.0 \)          & \( 0.1 \)          & --              & \FALSE    & \( 10 \) \\
        \mimvc & \( 256 \)  & \( 10^{-3} \)         & --                 & \( 0.1 \)          & \( 1.0 \)       & \FALSE    & \( 10 \) \\
        \mviic & \( 256 \)  & \( 10^{-3} \)         & --                 & \( 0.01 \)         & \( 1.0 \)       & \FALSE    & \( 10 \) \\
        \bottomrule
    \end{tabular}
}

%% file: inputs/abstract.tex
Self-supervised learning is a central component in recent approaches to deep multi-view clustering (MVC).
However, we find large variations in the development of self-supervision-based methods for deep MVC, potentially slowing the progress of the field.
To address this, we present \fwName, a unified framework for deep MVC that includes many recent methods as instances.
We leverage our framework to make key observations about the effect of self-supervision, and in particular, drawbacks of aligning representations with contrastive learning.
Further, we prove that contrastive alignment can negatively influence cluster separability, and that this effect becomes worse when the number of views increases.
Motivated by our findings, we develop several new \fwName instances with new forms of self-supervision.
We conduct extensive experiments and find that
(i) in line with our theoretical findings, contrastive alignments decreases performance on datasets with many views;
(ii) all methods benefit from some form of self-supervision;
and (iii) our new instances outperform previous methods on several datasets.
Based on our results, we suggest several promising directions for future research.
To enhance the openness of the field, we provide an open-source implementation of \fwName, including recent models and our new instances.
Our implementation includes a consistent evaluation protocol, facilitating fair and accurate evaluation of methods and components%
\footnote{Code: \githubLink[\footnotesize]}.

%% file: inputs/intro.tex
\begin{figure}
    \centering
    \includegraphics[width=\columnwidth]{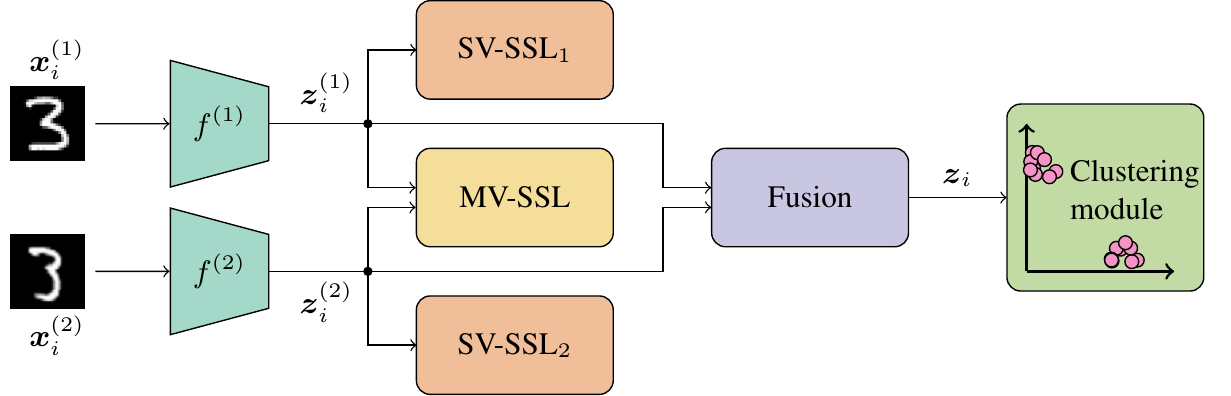}
    \caption{Overview of the \fwName framework for a two-view dataset. Different colors denote different components. The framework is generalizable to an arbitrary number of views by adding more view specific encoders (\( f \)) and SV-SSL blocks.}
    \label{fig:framework}
\end{figure}

Multi-view clustering (MVC) generalizes the clustering task to data where the instances to be clustered are observed through multiple views, or by multiple modalities.
In recent years, deep learning architectures have seen widespread adoption in MVC, resulting in the \emph{deep MVC} subfield.
Methods developed within this subfield have shown state-of-the-art clustering performance on several multi-view datasets~\cite{zhouEndtoEndAdversarialAttentionNetwork2020,trostenReconsideringRepresentationAlignment2021,xuMultiVAELearningDisentangled2021,maoDeepMutualInformation2021,wangAdversarialMultiviewClustering2022,wangSelfSupervisedInformationBottleneck2022}, largely outperforming traditional, non-deep-learning-based methods~\cite{zhouEndtoEndAdversarialAttentionNetwork2020}.

Despite these promising developments, we identify significant drawbacks with the current state of the field.
Self-supervised learning (SSL) is a crucial component in many recent methods for deep MVC~\cite{zhouEndtoEndAdversarialAttentionNetwork2020,trostenReconsideringRepresentationAlignment2021,xuMultiVAELearningDisentangled2021,maoDeepMutualInformation2021,wangAdversarialMultiviewClustering2022,wangSelfSupervisedInformationBottleneck2022}.
However, the large number of methods, all with unique components and arguments about how they work, makes it challenging to identify clear directions and trends in the development of new components and methods.
Methodological research in deep MVC thus lacks foundation and consistent directions for future advancements.
This effect is amplified by large variations in implementation and evaluation of new methods.
Architectures, data preprocessing and data splits, hyperparameter search strategies, evaluation metrics, and model selection strategies all vary greatly across publications, making it difficult to properly compare methods from different papers.
To address these challenges, we present a unified framework for deep MVC, coupled with a rigorous and consistent evaluation protocol, and an open-source implementation.
Our main contributions are summarized as follows:

\customparagraph{(1)~~\fwName framework.}
    Despite the variations in the development of new methods, we recognize that the majority of recent methods for deep MVC can be decomposed into the following fixed set of components:
    \begin{enumerate*}[label=(\roman*)]
        \item view-specific encoders;
        \item single-view SSL;
        \item multi-view SSL;
        \item fusion; and
        \item clustering module.
    \end{enumerate*}
    The \fwName framework (Figure~\ref{fig:framework}) is obtained by organizing these components into a unified deep MVC model.
    Methods from previous work can thus be regarded as \emph{instances} of \fwName.

\customparagraph{(2)~~Theoretical insight on alignment and number of views.}
    Contrastive alignment of view-specific representations is an MV-SSL component that has demonstrated state-of-the-art performance in deep MVC~\cite{trostenReconsideringRepresentationAlignment2021}.
    We study a simplified case of deep MVC, and find that contrastive alignment can only decrease the number of separable clusters in the representation space.
    Furthermore, we show that this potential negative effect of contrastive alignment becomes worse when the number of views in the dataset increases.

\customparagraph{(3)~~New instances of \fwName.}
    Inspired by initial findings from the \fwName framework, and our theoretical findings on contrastive alignment, we develop \( 6 \) new instances of \fwName, which outperform current state-of-the-art methods on several multi-view datasets.
    The new instances include both novel and well-known types of self-supervision, fusion and clustering modules.

\customparagraph{(4)~~Open-source implementation of \fwName and evaluation protocol.}
    We provide an open-source implementation of \fwName, including several recent methods and our new instances.
    The implementation includes a shared evaluation protocol for all methods, and all datasets used in the experimental evaluation.
    By making the datasets and our implementation openly available, we aim to facilitate simpler development of new methods, as well as rigorous and accurate comparisons between methods and components.

\customparagraph{(5)~~Evaluation of methods and components.}
    We use the implementation of \fwName to evaluate and compare several recent state-of-the-art methods and SSL components -- both against each other, and against our new instances.
    In our experiments, we both provide a consistent evaluation of methods in deep MVC, and systematically analyze several SSL-based components -- revealing how they behave under different experimental settings.

\customparagraph{The main findings from our work are:}
    \begin{itemize}[topsep=0pt, noitemsep, leftmargin=*]
        \item We show that aligning view-specific representations can have a negative impact on cluster separability, especially when the number of views becomes large.
            In our experiments, we find that contrastive alignment of view-specific representations works well for datasets with few views, but \emph{significantly degrades performance when the number of views increases}.
            Conversely, we find that maximization of mutual information performs well with many views, while not being as strong with fewer views.
        \item All methods included in our experiments benefit from at least one form of SSL.
            In addition to contrastive alignment for few views and mutual information maximization for many views, we find that autoencoder-style reconstruction improves overall performance of methods.
        \item Properties of the datasets, such as class (im)balance and the number of views, heavily impact the performance of current MVC approaches.
            There is thus not a single ``state-of-the-art'' -- it instead depends on the datasets considered.
        \item Results reported by the original authors differ significantly from the performance of our re-implementation for some baseline methods, illustrating the necessity of a unified framework with a consistent evaluation protocol.
    \end{itemize}

%% file: inputs/framework.tex
In this section we present the \fwName framework, its components and their purpose, and how they fit together.
This allows us to, in the next section, summarize recent work on deep MVC, and illustrate that the majority of recent methods can be regarded as instances of \fwName.

Suppose we have a multi-view dataset consisting of \( n \) instances and \( V \) views, and let \( \vec x_i^{(v)} \) be the observation of instance \( i \) through view \( v \).
The task of the \fwName framework is then to cluster the instances into \( k \) clusters, and produce cluster membership indicators \( \alpha_{ic} \in [0, 1] \), \( c = 1, \dots, k \).
The framework is illustrated in Figure~\ref{fig:framework}.
It consists of the following components.

\customparagraph{View-specific encoders.}
The framework is equipped with \( V \) deep neural network encoders \( f^{(1)}, \dots, f^{(V)} \), one for each view.
Their task is to produce the view-specific representations \( \vec z_i^{(v)} = f^{(v)}(\vec x_i^{(v)}) \) from the input data.

\customparagraph{Single-view self-supervised learning (SV-SSL).}
The SV-SSL component consists of a set of pretext tasks (auxiliary objectives) that are designed to aid the optimization of the view-specific encoders.
Specifically, the tasks should be designed to help the encoders learn representations that simplify the clustering task.
Each pretext task is specific to its designated view, and is isolated from all other views.

\customparagraph{Multi-view self-supervised learning (MV-SSL).}
MV-SSL is similar to SV-SSL -- they are both self-supervised modules whose goals are to help the encoders learn representations that are suitable for clustering.
However, MV-SSL leverages all views simultaneously in the pretext tasks, allowing the model to exploit information from all views simultaneously to learn better features.

\customparagraph{Fusion.}
This component combines view-specific representations into a shared representation for all views.
Fusion is typically done using a (weighted) average~\cite{liDeepAdversarialMultiview2019,trostenReconsideringRepresentationAlignment2021}, or by concatenation~\cite{huangMultiviewSpectralClustering2019,xinSelfSupervisedDeepCorrelational2021,xuMultiVAELearningDisentangled2021}.
More complex fusion modules using \eg attention mechanisms~\cite{zhouEndtoEndAdversarialAttentionNetwork2020}, are also possible.

\customparagraph{Clustering module (CM).}
The CM is responsible for determining cluster memberships based on view-specific or fused representations.
The CM can consist of a traditional clustering method, such as \( k \)-means~\cite{macqueenMethodsClassificationAnalysis1967} or Spectral Clustering~\cite{shiNormalizedCutsImage2000}.
Such CMs are applied to the fused representations after other components have been trained, resulting in a two-stage method that first learns fused representations, and then applies a clustering algorithm to these representations.

Alternatively, the CM can be integrated into the model~\cite{zhouEndtoEndAdversarialAttentionNetwork2020,trostenReconsideringRepresentationAlignment2021}, allowing it to be trained alongside other components, resulting in fused representations that are better suited for clustering.

\customparagraph{Loss functions and training.}
The loss functions for the models are specified by the SV-SSL, MV-SSL, and CM components.
To train the model, the terms arising from the different components can be minimized simultaneously or they can be minimized in an alternating fashion.
It is also possible with pre-training/fine-tuning setups where the model is pre-trained with one subset of the losses and fine-tuned with another subset of the losses.

We note that \fwName is a conceptual framework, and that a model is not necessarily completely described by a list of its \fwName components.
Consequently, it is possible for two models with similar \fwName components to have slightly different implementations.
This illustrates the importance of our open-source implementation of \fwName, which allows the implementation of a model to be completely transparent.

%% file: inputs/related-work.tex
\begin{table*}[!t]
    \centering
    {
        \small
        \renewcommand{\arraystretch}{0.85}
        \PreviousAndNewMethodsTable
    }
    \caption{
        Overview of selected methods from previous work (top) and proposed new instances (bottom), and their \fwName components. The complete table of previous methods is included in the supplementary.
        \textbf{Abbreviations:}
        ``\NOTINC'' = \NOTINCLong,
        ``\NOTSPEC'' = \NOTSPECLong,
        \AL~= \ALLong,
        \CAT = \CATLong,
        \CCA = \CCALong,
        \DDC = \DDCLong,
        \DEC = \DECLong,
        \SC = \SCLong,
        \SE = \SELong,
        \SR = \SRLong
    }
    \label{tab:previousMethods}
\end{table*}

Table~\ref{tab:previousMethods} shows selected recent methods for deep MVC (the full table can be found in the supplementary), categorized by its \fwName components, allowing for systematic comparisons between models\footnote{Note, here we limit our discussion to MVC approaches without missing data. While most of the theoretical and empirical results also generalize to the emerging incomplete MVC setting~\cite{xuAdversarialIncompleteMultiview2019,wenCDIMCnetCognitiveDeep2020,linCOMPLETERIncompleteMultiView2021}, we consider it out of scope of this work.}.

\customparagraph{View-specific encoders.}
    As can be seen in Table~\ref{tab:previousMethods}, all models use view-specific encoders to encode views into view-specific embeddings.
    Multi-layer perceptrons (MLPs) are usually used for vector data, while convolutional neural networks (CNNs) are used for image data.

\customparagraph{SV-SSL and MV-SSL.}
    Alongside the encoder network, many methods use decoders to reconstruct the original views from either the view-specific representations or the fused representation.
    The reconstruction task is the most common self-supervised pretext task, both for SV-SSL and for MV-SSL.
    In SV-SSL, the views are reconstructed from their respective view-specific representations, without any influence from the other views%
    ~\cite{wangDeepMultiViewRepresentation2015,abavisaniDeepMultimodalSubspace2018,tangDeepMultiviewSparse2018,sunSelfSupervisedDeepMultiView2019,zhangDeepMultimodalClustering2020,zhangEndToEndDeepMultimodal2020,zongMultimodalClusteringDeep2020,xuDeepEmbeddedMultiview2021}.
    In MV-SSL, it is common to either do
    \begin{enumerate*}[label=(\roman*)]
      \item cross view reconstruction, where all views are reconstructed from all view-specific representations~\cite{zhuMultiviewDeepSubspace2019}; or
      \item fused view reconstruction, where all views are reconstructed from the fused representation~\cite{zhuMultiviewDeepSubspace2019,liDeepAdversarialMultiview2019,yinSharedGenerativeLatent2020,wangAdversarialMultiviewClustering2022}.
    \end{enumerate*}

    Aligning distributions of view-specific representations is another MV-SSL pretext task that has been shown to produce representations suitable for clustering~\cite{zhouEndtoEndAdversarialAttentionNetwork2020}.
    However,~\cite{trostenReconsideringRepresentationAlignment2021} demonstrate that the alignment of representation distributions can be detrimental to the clustering performance -- especially in the presence of noisy or non-informative views.
    To avoid these drawbacks, they propose Simple MVC (SiMVC) and Contrastive MVC (CoMVC).
    In the former, the alignment is dropped altogether, whereas the latter includes a contrastive learning module that aligns the view-specific representations at the instance level, rather than at the distribution level.

\customparagraph{Clustering modules.}
    Many deep MVC methods use subspace-based clustering modules~\cite{abavisaniDeepMultimodalSubspace2018,zhuMultiviewDeepSubspace2019,sunSelfSupervisedDeepMultiView2019,wangSelfSupervisedInformationBottleneck2022}.
    These methods assume that representations, either view-specific or fused, can be decomposed into linear combinations of each other.
    Once determined, the self-representation matrix containing the coefficients for these linear combinations is used to compute an affinity matrix, which in turn is used as input to spectral clustering.
    This requires the full \( n \times n \) self-representation matrix available in memory, which is computationally prohibitive for datasets with a large number of instances.

    Other clustering modules have also been adapted to deep MVC. The clustering module from Deep Embedded Clustering (DEC)~\cite{xieUnsupervisedDeepEmbedding2016}, for instance, is used in several models~\cite{liDeepAdversarialMultiview2019,duDeepMultipleAutoEncoderBased2021,xinSelfSupervisedDeepCorrelational2021,xuDeepEmbeddedMultiview2021,wangAdversarialMultiviewClustering2022}.
    Recently, the Deep Divergence-Based Clustering (DDC)~\cite{kampffmeyerDeepDivergencebasedApproach2019} clustering module has been used in several state-of-the-art deep MVC models~\cite{zhouEndtoEndAdversarialAttentionNetwork2020,trostenReconsideringRepresentationAlignment2021}.
    In addition, some methods treat either the encoder output or the fused representation as cluster membership vectors~\cite{zhangEndToEndDeepMultimodal2020,maoDeepMutualInformation2021}.

    Lastly, some methods adopt a two-stage approach, where they first use the SSL components to learn representations, and then apply a traditional clustering method, such as \( k \)-means~\cite{huangMultiSpectralNetSpectralClustering2019,huangMultiviewSpectralClustering2019,zhangDeepMultimodalClustering2020,zongMultimodalClusteringDeep2020,xuMultiVAELearningDisentangled2021}, a Gaussian mixture model~\cite{yinSharedGenerativeLatent2020}, or spectral clustering~\cite{wangDeepMultiViewRepresentation2015}, on the trained representations.

%% file: inputs/motivating-experiments.tex
As can be seen in Table~\ref{tab:previousMethods}, SSL components are crucial in recent state-of-the-art methods for deep MVC.
Recent works have focused on aligning view-specific representations~\cite{zhouEndtoEndAdversarialAttentionNetwork2020,trostenReconsideringRepresentationAlignment2021}, and in particular, contrastive alignment~\cite{trostenReconsideringRepresentationAlignment2021}.
We study a simplified setting where, for each view, all observations in a cluster are located at the same point.
This allows us to prove that aligning view-specific representations has a negative impact on the cluster separability after fusion.
This is the same starting point as in~\cite{trostenReconsideringRepresentationAlignment2021}, but we extend the analysis to investigate contrastive alignment when the number of views increases.

\input{inputs/propositions.tex}
\propositionMinSeparbleClusters
\begin{proof}
    See~\cite{trostenReconsideringRepresentationAlignment2021}.
\end{proof}

According to Proposition~\ref{prop:prop1}, when the view-specific representations are perfectly aligned, the number of separable clusters after fusion, \( \kappa \), depends on the number of separable clusters in the \emph{least informative view} -- the view with the lowest \( k_v \).
The following propositions show what happens to \( \min \{ k_v \} \) when the number of views increases\footnote{The proofs of Propositions~\ref{prop:conditionalProbMinimum} and~\ref{prop:expectationMinimum} are given in the supplementary}.

\propositionConditionalProbMinimum
\propositionExpectationMinimum

Assuming the view-specific representations are perfectly aligned, Propositions~\ref{prop:conditionalProbMinimum} and~\ref{prop:expectationMinimum} show that:
\begin{enumerate*}[label=(\roman*)]
    \item Given a number of views, adding another view will, with probability \( 1 \), not increase \( \min \{k_v \} \).
    \item Among two datasets with the same distribution for the \( k_v \), the dataset with the \emph{smallest number of views} will have the highest expected value of \( \min \{k_v\} \).
\end{enumerate*}

In summary, we have shown that contrastive alignment-based models perform worse when the number of views in a dataset increases.
These findings are supported by the experimental results in Figure~\ref{fig:motivatingIncviews} and Table~\ref{tab:motivatingMNIST} which show that, when the number of views increases, the contrastive alignment-based model is outperformed by the model without any alignment.

\customparagraph{Alignment as a pretext task.}
    In contrast to our theoretical findings in the simplified case, Figure~\ref{fig:motivatingIncviews} and Table~\ref{tab:motivatingMNIST} show that contrastive alignment can sometimes be beneficial for the performance, particularly when the number of views is small.
    This is because alignment might be a good pretext task that helps the encoders learn informative representations, by learning to represent the information that is shared across views.
    However, we emphasize that this is only true when the number of views is small ($\le 4$ in Figure~\ref{fig:motivatingIncviews}), meaning that alignment should be used with caution when the number of views increases beyond this point.

\begin{figure}
\begin{floatrow}
\ffigbox[0.35\columnwidth]{%
    \centering
    \bgroup
    \def\figwidth{3.7cm}
    \def\figheight{4.2cm}
    \scriptsize
    \input{fig/increasing_views/motivating.tex}
    \egroup
}{%
  \caption{Clustering accuracy for an increasing number of views on Caltech7.}
  \label{fig:motivatingIncviews}
}
\capbtabbox[0.5\columnwidth]{%
    \bgroup
    \tableFontSize
    \setlength{\tabcolsep}{.9mm}
    \input{tab/motivating-mnist.tex}

    \egroup
}{%
  \caption{Clustering accuracies on datasets with varying number of views.}
  \label{tab:motivatingMNIST}
}
\end{floatrow}
\end{figure}

%% file: inputs/propositions.tex
\def\propositionMinSeparbleClusters{
    \begin{proposition}[Adapted from~\cite{trostenReconsideringRepresentationAlignment2021}]
        \label{prop:prop1}
        Suppose the dataset consists of \( n \) instances, \( V \) views, and \( k \) ground-truth clusters, and that view-specific representations are computed with view-specific encoders as \( \vec z^{(v)}_i = f^{(v)}(\vec x^{(v)}_i) \).
        Furthermore, assume that:
        \begin{itemize}
            \item For all \( v \in \{1, \dots V\} \) and \( j \in \{1, \dots, k \} \),
                \begin{align}
                    \forall i \in \cl C_j, \vec x^{(v)}_i = \vec c^{(v)} \in \{ \vec c_1^{(v)}, \dots, \vec c_{k_v}^{(v)} \}
                \end{align}
                where \( \cl C_j \) is the set of indices for instances in cluster \( j \), and \( k_v \in \{ 1, \dots, k \} \) is the number of separable clusters in view \( v \).
            \item Representations are fused as \( \vec z_i = \sum_{v=1}^{V} w_v \vec z_i^{(v)} \) where \( w_1, \dots, w_V \) are all unique.
            \item For all \( j \in \{1, \dots, k \} \),
                \begin{align}
                    \forall i \in \cl C_j, \vec z_i = \vec z^{\star} \in \{ \vec z_1^\star, \dots, \vec z_\kappa^\star \}
                \end{align}
        \end{itemize}
        Then if \( \vec z_i^{(1)} = \dots = \vec z_i^{(V)} \) (perfectly aligned view-specific representations),
        \begin{align}
            \kappa = \min \{k, (\min\limits_{v=1,\dots,V} \{k_v \})^V \}
        \end{align}
    \end{proposition}
}

\def\propositionConditionalProbMinimum{
    \begin{proposition}
         \label{prop:conditionalProbMinimum}
         Suppose \( k_v, v \in \naturals \) are random variables taking values in \( \{1, \dots, k \} \).
         Then, for any \( V \ge 1 \),
         \begin{align}
             \Prob \{ \min\limits_{v=1, \dots, V+1} \lrc{k_v} \le \min\limits_{v=1, \dots, V} \lrc{k_v} ~\Big|~ k_1, \dots, k_V \} = 1
         \end{align}
    \end{proposition}
}
\def\propositionExpectationMinimum{
    \begin{proposition}
        \label{prop:expectationMinimum}
        Suppose \( k_v, v \in \naturals \) are iid.\ random variables taking values in \( \{1, \dots, k \} \).
        Then, for any \( V \ge 1 \),
        \begin{align}
            \E (\min\limits_{v=1, \dots, V+1} \lrc{k_v}) \le \E (\min\limits_{v=1, \dots, V} \lrc{k_v})
        \end{align}
    \end{proposition}
}

%% file: fig/increasing_views/motivating.tex
\begin{tikzpicture}

\definecolor{chocolate217952}{RGB}{217,95,2}
\definecolor{darkcyan27158119}{RGB}{27,158,119}
\definecolor{darkgray176}{RGB}{176,176,176}
\definecolor{lightgray204}{RGB}{204,204,204}

\begin{groupplot}[group style={group size=1 by 1}]
\nextgroupplot[
height=\figheight,
legend cell align={left},
legend style={
  fill opacity=0.8,
  draw opacity=1,
  text opacity=1,
  at={(0.03,0.03)},
  anchor=south west,
  draw=lightgray204
},
tick align=outside,
tick pos=left,
width=\figwidth,
x grid style={darkgray176},
xlabel={Number of views},
xmajorgrids,
xmin=1.8, xmax=6.2,
xtick style={color=black},
y grid style={darkgray176},
ylabel={ACC},
ymajorgrids,
ymin=0.2, ymax=0.5,
ytick style={color=black}
]
\path [draw=darkcyan27158119, fill=darkcyan27158119, opacity=0.2]
(axis cs:2,0.375364035135541)
--(axis cs:2,0.357336074346271)
--(axis cs:3,0.38645462505398)
--(axis cs:4,0.366926940756274)
--(axis cs:5,0.350017003154158)
--(axis cs:6,0.371804667381619)
--(axis cs:6,0.393459903331424)
--(axis cs:6,0.393459903331424)
--(axis cs:5,0.374542005921007)
--(axis cs:4,0.396980789823102)
--(axis cs:3,0.391021626071356)
--(axis cs:2,0.375364035135541)
--cycle;

\path [draw=chocolate217952, fill=chocolate217952, opacity=0.2]
(axis cs:2,0.391650448011561)
--(axis cs:2,0.372257282567815)
--(axis cs:3,0.317603556653463)
--(axis cs:4,0.321921504793895)
--(axis cs:5,0.366895840213459)
--(axis cs:6,0.381128247307615)
--(axis cs:6,0.428912465048952)
--(axis cs:6,0.428912465048952)
--(axis cs:5,0.375302269413311)
--(axis cs:4,0.361931942643392)
--(axis cs:3,0.358108790377177)
--(axis cs:2,0.391650448011561)
--cycle;

\addplot [semithick, darkcyan27158119]
table {%
2 0.366350054740906
3 0.388738125562668
4 0.381953865289688
5 0.362279504537582
6 0.382632285356522
};
\addlegendentry{w/ align}
\addplot [semithick, chocolate217952, dashed]
table {%
2 0.381953865289688
3 0.33785617351532
4 0.341926723718643
5 0.371099054813385
6 0.405020356178284
};
\addlegendentry{w/o align}
\end{groupplot}

\end{tikzpicture}

%% file: tab/motivating-mnist.tex
\rowcolors{2}{gray!25}{white}
\begin{tabular}{ccc}\toprule
    Dataset & w/o align & w/ align \\ \midrule
    \ST[c]{Edge-\\MNIST\\(2 views)} & \MTC{0.89} & \MTC{\BEST{0.97}} \\
    \ST[c]{Caltech7\\(6 views)} & \MTC{\BEST{0.41}} & \MTC{0.38} \\
    \ST[c]{Patched-\\MNIST\\(12 views)} & \MTC{\BEST{0.84}} & \MTC{0.73} \\
    \bottomrule
\end{tabular}

%% file: inputs/new-instances.tex
With our new instances of \fwName, we aim to further analyze and address the many-views-issue with contrastive alignment highlighted above, as well as to investigate the effect of other SSL components.
In addition to alignment of view-specific representations~\cite{zhouEndtoEndAdversarialAttentionNetwork2020,duDeepMultipleAutoEncoderBased2021,trostenReconsideringRepresentationAlignment2021}, we identify reconstruction~\cite{wangDeepMultiViewRepresentation2015,abavisaniDeepMultimodalSubspace2018,liDeepAdversarialMultiview2019} and mutual information maximization~\cite{jiInvariantInformationClustering2019,wangSelfSupervisedInformationBottleneck2022} to be promising directions for the new instances.
Maximizing mutual information is particularly interesting, as it enables the view-specific encoders to represent the information which is shared across views, without explicitly forcing the view-specific representations to be aligned. 
Furthermore, we recognize that simple baselines with few or no SSL components -- exemplified by SiMVC~\cite{trostenReconsideringRepresentationAlignment2021} -- might perform similarly to more complicated methods, while being significantly easier to implement and faster to train.
It is therefore crucial to include such methods in an experimental evaluation, in order to properly determine whether additional SSL-based components are beneficial for the models' performance.
Finally, our overview of recent work shows that both traditional clustering modules (\eg \( k \)-means) and deep learning-based clustering modules (\eg DDC) are commonly used in deep MVC.

In total, we develop \( 6 \) new \fwName instances in \( 3 \) categories.
The new instances are summarized in Table~\ref{tab:previousMethods}.
Evaluating these instances and several methods from recent work, allows us to accurately evaluate methods and components, and investigate how they behave for datasets with varying characteristics.

\customparagraph{Simple baselines:}
    \textbf{\saekm} has view-specific autoencoders (AEs) with a mean-squared-error (MSE) loss
    \begin{align}
        \label{eq:mseLoss}
        \SVSSLLoss_{\text{Reconstruction}} = \frac{1}{nV} \sums{i=1}{n}\sums{v=1}{V} ||\vec x_i^{(v)} - \hat{\vec x}_i^{(v)} ||^2
    \end{align}
    as its SV-SSL task.
    The views are fused by concatenation and the concatenated representations are clustered using \( k \)-means after the view-specific autoencoders have been trained.
    \textbf{\sae} uses view-specific autoencoders with an MSE loss (Eq.~\eqref{eq:mseLoss}) as its SV-SSL task.
    The views are fused using a weighted sum and the fused representations are clustered using the DDC clustering module~\cite{kampffmeyerDeepDivergencebasedApproach2019}.

\customparagraph{Contrastive alignment-based:}
    \textbf{\caekm} extends \saekm~with a contrastive loss on the view-specific representations.
    We use the multi-view generalization of the NT-Xent (contrastive) loss by Trosten \etal~\cite{trostenReconsideringRepresentationAlignment2021}, without the ``other clusters'' negative sampling
    \begin{align}
        & \MVSSLLoss_{\text{Contrastive}} = \frac{1}{nV(V-1)} \sums{i=1}{n}\sums{v=1}{V}\sums{u=1}{V} \mathds{1}_{\{u \neq v\}}\ \ell_i^{(uv)}, \\
        & \ell_i^{(uv)} = - \log \frac{\exp(s_{ii}^{(uv)})}{ \sum_{s' \in \text{Neg}(\vec z_i^{(u)}, \vec z_i^{(v)})} \exp(s')}
    \end{align}
    and \( s_{ij}^{(uv)} = \frac{1}{\tau} \frac{\vec z_i^{u} \cdot \vec z_j^{(v)}}{ ||\vec z_i^{u}|| \cdot ||\vec z_j^{(v)}||} \) denotes the cosine similarity between \( \vec z_i^{u} \) and \( \vec z_j^{v} \).
    The set \( \text{Neg}(\vec z_i^{(u)}, \vec z_i^{(v)}) \) is the set of similarities of negative pairs for the positive pair \( (\vec z_i^{(u)}, \vec z_i^{(v)}) \),
    which consists of \( s^{(uv)}_{ij} \), \( s^{(uu)}_{ij} \), and \( s^{(vv)}_{ij} \), for all \( j \neq i \).
    \( \tau \) is a hyperparameter, which we set to \( 0.1 \) for all experiments.
    \textbf{\cae} extends \sae~using the same generalized NT-Xent contrastive loss on the view-specific representations.

\customparagraph{Mutual information-based:}
    \textbf{\mimvc} maximizes the mutual information (MI) between the view-specific representations, using the MI loss from Invariant Information Clustering (IIC)~\cite{jiInvariantInformationClustering2019}\footnote{The supplementary includes a brief overview of the connection between InfoDDC and contrastive alignment.}.
    The MI maximization is regularized by also maximizing the entropy of view-specific representations.
    The view-specific representations are fused using a weighted sum, and the fused representations are clustered using DDC~\cite{kampffmeyerDeepDivergencebasedApproach2019}.
    \textbf{\mviic} is a multi-view generalization of IIC~\cite{jiInvariantInformationClustering2019}, where cluster assignments are computed for each of the view-specific representations.
    The MI between pairs of these view-specific cluster assignments is then maximized using the information maximization loss from IIC.
    In order to get a final shared cluster assignment for all views, the view-specific cluster assignments are concatenated and clustered using \( k \)-means.
    As in IIC, this model includes \( 5 \) over-clustering heads as its MV-SSL task.
    In both \mimvc and \mviic, we generalize the loss from IIC to an arbitrary number of views:
    \begin{align}
        \nonumber
        \MVSSLLoss_{\text{MI}} &= \frac{2}{V(V-1)} \sums{u=1}{V-1}\sums{v=u+1}{V} - \Big( \underbrace{I(\vec Z^{(u)}, \vec Z^{(v)})}_{\text{mutual information}} \\
        & + (\lambda - 1) \underbrace{(H(\vec Z^{(u)}) + H(\vec Z^{(v)}))}_{\text{entropy regularization}} \Big)
    \end{align}
    where the summands are computed as
    \begin{align}
        \nonumber
        & I(\vec Z^{(u)}, \vec Z^{(v)}) + (\lambda - 1) (H(\vec Z^{(u)}) + H(\vec Z^{(v)})) \\
        & = - \sums{a=1}{D}\sums{b=1}{D} \vec P^{(uv)}_{ab} \log \frac{\vec P^{(uv)}_{ab}}{(\vec P^{(u)}_a \vec P^{(v)}_b)^\lambda},
    \end{align}
    where \( D \) denotes the dimensionality of the view-specific representations.
    \( \lambda \) is a hyperparameter that controls the strength of the entropy regularization.
    We set \( \lambda = 10 \) for \mimvc, and \( \lambda = 1.5 \) for \mviic.
    The joint distribution \( \vec P^{(uv)} \) is estimated by first computing \( \tilde{\vec P}^{(uv)} = \frac{1}{n} \sum_{i=1}^{n} \vec z^{(u)}_i (\vec z^{(v)}_i)\T \),
    and then symmetrizing it \( \vec P^{(uv)} = \frac{1}{2} (\tilde{\vec P}^{(uv)} + (\tilde{\vec P}^{(uv)})\T)\).
    We assume that each view-specific representation is normalized such that its elements sum to one, and are all non-negative.
    The marginals \( \vec P^{(u)} \) and \( \vec P^{(v)} \) are obtained by summing over the rows and columns of \( \vec P^{(uv)} \), respectively.

%% file: inputs/experiments.tex
In this section we provide a rigorous evaluation of methods and their \fwName components.
Inspired by the initial findings in Section~\ref{sec:motivatingExperiments} and our overview of recent methods in Section~\ref{sec:relatedWork}, we focus mainly on the SSL and CM components in our evaluation.
We found these components to be most influential on the methods' performance.
For completeness, we include experiments with different fusion and CM components in the supplementary.

\subsection{Setup}
\label{subsec:setup}
    \thisfloatsetup{floatrowsep=quad}
    \begin{figure*}
        \begin{floatrow}
        \capbtabbox[0.8\textwidth]{%
            \setlength{\tabcolsep}{.6mm}
            \tableFontSize\undoCaptionSep\renewcommand{\arraystretch}{0.85}
            \input{tab/benchmark/agg/merged.tex}
        }{%
          \caption{
            Aggregated evaluation results for the dataset groups. Models are sorted from lowest to highest by average Z-score for each group. Higher Z-scores indicate better clusterings. Our new instances are \underline{underlined}.
            \textbf{Abbreviations:}
            BL = Simple baseline,
            CA = Contrastive alignment,
            DDC = Deep divergence-based clustering
            MI = Mutual information,
            \( \bar Z \) = Average Z-score for group.
        }
        \label{tab:aggBenchmark}
        }
        \ffigbox[0.15\textwidth]{%
          {\scriptsize
            \flushleft
            \def\figwidth{3.7cm}
            \def\figheight{4.5cm}
            \input{fig/increasing_views/new_instances.tex}}

        }{%
          \caption{Accuracies on Caltech7 with increasing number of views.}
        \label{fig:incviews}
        }
        \end{floatrow}
    \end{figure*}

    \customparagraph{Baselines.}
        In addition to the new instances presented in Section~\ref{sec:newVariations}, we include \( 6 \) baseline models from previous work in our experiments.
        The following baseline models were selected to include a diverse set of framework components in the evaluation:
        \begin{enumerate*}[label=(\roman*)]
            \item Deep Multimodal Subspace Clustering (DMSC)~\cite{abavisaniDeepMultimodalSubspace2018};
            \item Multi-view Spectral Clustering Network (MvSCN)~\cite{huangMultiviewSpectralClustering2019};
            \item End-to-end Adversarial-attention Multimodal Clustering (EAMC)~\cite{zhouEndtoEndAdversarialAttentionNetwork2020};
            \item Simple Multi-View Clustering (SiMVC)~\cite{trostenReconsideringRepresentationAlignment2021};
            \item Contrastive Multi-View Clustering (CoMVC)~\cite{trostenReconsideringRepresentationAlignment2021};
            \item Multi-view Variational Autoencoder (Multi-VAE)~\cite{xuMultiVAELearningDisentangled2021}.
        \end{enumerate*}

        As can be seen in Table~\ref{tab:previousMethods}, this collection of models includes both reconstruction-based and alignment-based SSL, as well as traditional (\( k \)-means and spectral) and deep learning-based CMs.
        They also include several fusion strategies and encoder networks.
        Section~\ref{subsec:ablation} includes an ablation study that examines the influence of SSL components in these models.

    \customparagraph{Datasets.}
        We evaluate the baselines and new instances on \( 8 \) widely used benchmark datasets for deep MVC.
        We prioritize datasets that were also used in the original publications for the selected baselines.
        Not only does this result in a diverse collection of datasets common in deep MVC -- it also allows us to compare the performance of our implementations to what was reported by the original authors.
        The results of this comparison are given in the supplementary.

        The following datasets are used for evaluation:
        \begin{enumerate*}[label=(\roman*)]
            \item \textbf{NoisyMNIST\,/\,NoisyFashion}: A version of MNIST~\cite{lecunGradientbasedLearningApplied1998}\,/\,FashionMNIST~\cite{xiaoFashionMNISTNovelImage2017} where the first view contains the original image, and the second view contains an image sampled from the same class as the first image, with added Gaussian noise (\( \sigma = 0.2 \)).
            \item \textbf{EdgeMNIST\,/\,EdgeFashion}: Another version of MNIST\,/\,FashionMNIST where the first view contains the original image, and the second view contains an edge-detected version of the same image.
            \item \textbf{COIL-20}: The original COIL-20~\cite{neneColumbiaObjectImage1996} dataset, where we randomly group the images of each object into groups of size \( 3 \), resulting in a \( 3 \)-view dataset.
            \item \textbf{Caltech7\,/\,Caltech20}: A subset of the Caltech101~\cite{fei-feiLearningGenerativeVisual2007} dataset including \( 7 \)\,/\,\( 20 \) classes.
                We use the \( 6 \) different features extracted by Li \etal~\cite{liLargeScaleMultiViewSpectral2015}, resulting in a \( 6 \)-view dataset\footnote{The list of classes and feature types is included in the supplementary.}.
            \item \textbf{PatchedMNIST}: A subset of MNIST containing the first three digits, where views are extracted as \( 7 \times 7 \) non-overlapping patches of the original image.
                The corner patches are dropped as they often contain little information about the digit, resulting in a dataset with \( 12 \) views.
                Each patch is resized to \( 28 \times 28 \).
        \end{enumerate*}

        All views are individually normalized so that the values lie in \( [0, 1] \).
        Following recent work on deep MVC, we train and evaluate on the full datasets~\cite{zhouEndtoEndAdversarialAttentionNetwork2020,trostenReconsideringRepresentationAlignment2021,xuMultiVAELearningDisentangled2021,maoDeepMutualInformation2021}.
        More dataset details are provided in the supplementary.

    \customparagraph{Hyperparameters.}
        The baselines use the hyperparameters reported by the original authors, because
        \begin{enumerate*}[label=(\roman*)]
            \item it is not feasible for us to tune hyperparameters individually for each model on each dataset; and
            \item it is difficult to tune hyperparameters in a realistic clustering setting due to the lack of labeled validation data.
        \end{enumerate*}
        For each method, the same hyperparameter configuration is used for all datasets.
        
        New instances use the same hyperparameters as for the baselines wherever possible\footnote{Hyperparameters for all models are listed in the supplementary.}.
        Otherwise, we set hyperparameters such that loss terms have the same order of magnitude, and such that the training converges.
        We refrain from any hyperparameter tuning that includes the dataset labels to keep the evaluation fair and unsupervised.
        We include a hyperparameter sweep in the supplementary, in order to assess the new instances' sensitivity to changes in their hyperparameter.
        However, we emphasize that the results of this sweep were \emph{not} used to select hyperparameters for the new instances.
        All models use the same encoder architectures and are trained for \( 100 \) epochs with the Adam optimizer~\cite{kingmaAdamMethodStochastic2015}.

    \customparagraph{Evaluation protocol.}
        We train each model from \( 5 \) different initializations.
        Then we select the run that resulted in the lowest value of the loss and report the performance metrics from that run, following~\cite{kampffmeyerDeepDivergencebasedApproach2019,trostenReconsideringRepresentationAlignment2021}.
        This evaluation protocol is both fully unsupervised, and is not as impacted by poorly performing runs, as for instance the mean performance of all runs.
        The uncertainty of the performance metric under this model selection protocol is estimated using bootstrapping\footnote{Details on uncertainty computations are included in the supplementary.}.
        We measure clustering performance with the accuracy (ACC) and normalized mutual information (NMI).
        Both metrics are bounded in \( [0, 1] \), and higher values correspond to better performing models, with respect to the ground truth labels.

\subsection{Evaluation results}
\label{subsec:benchmarkResults}
    To emphasize the findings from our experiments, we compute the average Z-score for each model, for \( 4 \) groups of datasets\footnote{Results for all methods/datasets are included in the supplementary.}.
    Z-scores are calculated by subtracting the mean and dividing by the standard deviation of results, per dataset and per metric.
    Table~\ref{tab:aggBenchmark} shows Z-scores for the groups:
    \begin{enumerate*}[label=(\roman*)]
        \item \textbf{All datasets}.
        \item \textbf{Random pairings:} Datasets generated by randomly pairing within-class instances to synthesize multiple views (NoisyMNIST, NoisyFashion, COIL-20).
        \item \textbf{Many views:} Datasets with many views (Caltech7, Caltech20, PatchedMNIST).
        \item \textbf{Balanced vs.\ imbalanced:} Datasets with balanced classes (NoisyMNIST, NoisyFashion, EdgeMNIST, EdgeFashion, COIL-20, PatchedMNIST) vs.\ datasets with imbalanced classes (Caltech7, Caltech20).
    \end{enumerate*}
    Our main experimental findings are:

    \textbf{Dataset properties significantly impact the performance of methods.}
    We observe that the ranking of methods varies significantly based on dataset properties, such as the number of views (Table~\ref{tab:aggBenchmark}c) and class (im)balance (Table~\ref{tab:aggBenchmark}d).
    Hence, there is not a single ``state-of-the-art'' for all datasets.

    \textbf{Our new instances outperform previous methods.}
    In Table~\ref{tab:aggBenchmark}a we see that the simple baselines perform remarkably well, when compared to the other, more complex methods.
    This highlights the importance of including simple baselines like these in the evaluation.
    Table~\ref{tab:aggBenchmark}a shows that \cae overall outperforms the other methods, and on datasets with many views (Table~\ref{tab:aggBenchmark}c) we find that \mimvc and \mviic outperform the others by a large margin.

    \textbf{Maximization of mutual information outperforms contrastive alignment on datasets with many views.}
    Contrastive alignment-based methods show good overall performance, but they struggle when the number of views becomes large (Table~\ref{tab:aggBenchmark}c).
    This holds for both baseline methods (as observed in Section~\ref{sec:motivatingExperiments}), and the new instances.
    As in Section~\ref{sec:motivatingExperiments}, we hypothesize that this is due to issues with representation alignment, where the presence of less informative views is more likely when the number of views becomes large.
    Contrastive alignment attempts to align view-specific representations to this less informative view, resulting in clusters that are harder to separate in the representation space.
    This is further verified in Figures~\ref{fig:motivatingIncviews} and~\ref{fig:incviews}, illustrating a decrease in performance on Caltech7 for contrastive alignment-based models with \( 5 \) or \( 6 \) views.
    Models based on maximization of mutual information do not have the same problem.
    We hypothesize that this is because maximizing mutual information still allows the view-specific representations to be different, avoiding the above issues with alignment.
    The MI-based models also include regularization terms that maximize the entropy of view-specific representations, preventing the representations from collapsing to a single value.

    \textbf{Contrastive alignment works particularly well on datasets consisting of random pairings (Table~\ref{tab:aggBenchmark}b).}
    In these datasets, the class label is the only thing the views have in common.
    Contrastive alignment, \ie learning a shared representation for all pairs within a class, thus asymptotically amounts to learning a unique representation for each class, making it easier for the CM to separate between classes.

    \textbf{The DDC CM performs better than the other CMs on balanced datasets.}
    With the DDC CM, the models are end-to-end trainable -- jointly optimizing all components in the model.
    The view-specific representations can thus be adapted to suit the CM, potentially improving the clustering result.
    DDC also has an inherent bias towards balanced clusters~\cite{kampffmeyerDeepDivergencebasedApproach2019}, which helps produce better clusterings when the ground truth classes are balanced.

    \textbf{Reproducibility of original results.}
    During our experiments we encountered issues with reproducibility with several of the methods from previous work.
    In the supplementary we include a comparison between our results and those reported by the original authors of the methods from previous work.
    We find that most methods use different network architectures and evaluation protocols in the original publications, making it difficult to accurately compare performance between methods and their implementations.
    This illustrates the difficulty of reproducing and comparing results in deep MVC, highlighting the need for a unified framework with a consistent evaluation protocol and an open-source implementation.

\subsection{Effect of SSL components}
\label{subsec:ablation}
    \begin{table}[t]
        \centering
        \setlength{\tabcolsep}{1mm}
        \tableFontSize\undoCaptionSep\renewcommand{\arraystretch}{0.85}
        \input{tab/ablation/both_ssl.tex}
        \caption{Accuracies from ablation studies with SSL components.}
        \label{tab:ablationSSL}
    \end{table}
    Table~\ref{tab:ablationSSL} shows the results of ablation studies with the SV-SSL and MV-SSL components.
    These results show that having at least one form of SSL is beneficial for the performance of all models, with the exception being \sae/\cae, which on Caltech7 performs best without any self-supervision.
    We suspect that this particular result is due to the issues with many views and class imbalance discussed in Section~\ref{subsec:benchmarkResults}.
    Further, we observe that having both forms of SSL is not always necessary.
    For instance is there no difference with and without SV-SSL for \cae and \caekm, both of which include contrastive alignment-based MV-SSL.
    Lastly, we note that contrastive alignment-based MV-SSL decreases performance on Caltech7 for most models.
    This is consistent with our theoretical findings in Section~\ref{sec:motivatingExperiments}, as well as the results in Section~\ref{subsec:benchmarkResults} and in Figures~\ref{fig:motivatingIncviews} and~\ref{fig:incviews} -- illustrating that contrastive alignment is not suitable for datasets with a large number of views.

%% file: tab/benchmark/agg/merged.tex
\bgroup

\let\oldsae\sae
\let\oldcae\cae
\let\oldsaekm\saekm
\let\oldcaekm\caekm
\let\oldmviic\mviic
\let\oldmimvc\mimvc

\def\formatModel#1{\underline{#1}}

\def\sae{\formatModel{\oldsae}}
\def\cae{\formatModel{\oldcae}}
\def\saekm{\formatModel{\oldsaekm}}
\def\caekm{\formatModel{\oldcaekm}}
\def\mviic{\formatModel{\oldmviic}}
\def\mimvc{\formatModel{\oldmimvc}}

\begin{tabular}{lcr|lcr|lccr|lcrr}
    \multicolumn{3}{c}{(a) \bfseries All datasets} & \multicolumn{3}{c}{(b) \bfseries Random pairings} & \multicolumn{4}{c}{(c) \bfseries Many views} & \multicolumn{4}{c}{(d) \bfseries Balanced vs.\ imbalanced} \\
    \toprule
    Model    & BL     & \( \bar Z \)& Model  & CA      & \( \bar Z \)& Model     & MI      & CA      & \( \bar Z \)& Model    & DDC    & \( \bar Z_\text{bal} \) & \( \bar Z_\text{imb} \)    \\ \midrule
    \mvscn   & \FALSE & \MTC{-2.23} & \mvscn   & \FALSE & \MTC{-2.49}& \mvscn   & \FALSE   & \FALSE  & \MTC{-1.78} & \mvscn   & \FALSE & \MTC{-2.41}& \MTC{-1.78} \\
    \caekm   & \FALSE & \MTC{-0.32} & \dmsc    & \FALSE & \MTC{-0.54}& \caekm   & \FALSE   & \TRUE   & \MTC{-0.83} & \dmsc    & \FALSE & \MTC{-0.39}& \MTC{0.45}  \\
    \eamc    & \FALSE & \MTC{-0.22} & \mimvc   & \FALSE & \MTC{-0.41}& \eamc    & \FALSE   & \FALSE  & \MTC{-0.75} & \mimvc   & \TRUE  & \MTC{-0.13}& \MTC{1.18}  \\
    \dmsc    & \FALSE & \MTC{-0.11} & \eamc    & \FALSE & \MTC{-0.17}& \sae     & \FALSE   & \FALSE  & \MTC{-0.36} & \eamc    & \TRUE  & \MTC{0.00} & \MTC{-0.75} \\
    \saekm   & \TRUE  & \MTC{0.16}  & \mviic   & \FALSE & \MTC{0.05} & \comvc   & \FALSE   & \TRUE   & \MTC{-0.33} & \mviic   & \FALSE & \MTC{0.01} & \MTC{1.06}  \\
    \mimvc   & \FALSE & \MTC{0.20}  & \saekm   & \FALSE & \MTC{0.11} & \simvc   & \FALSE   & \FALSE  & \MTC{-0.12} & \saekm   & \FALSE & \MTC{0.03} & \MTC{0.56}  \\
    \sae     & \TRUE  & \MTC{0.26}  & \mvae    & \FALSE & \MTC{0.32} & \saekm   & \FALSE   & \FALSE  & \MTC{0.23}  & \caekm   & \FALSE & \MTC{0.08} & \MTC{-1.54} \\
    \simvc   & \TRUE  & \MTC{0.27}  & \simvc   & \FALSE & \MTC{0.35} & \cae     & \FALSE   & \TRUE   & \MTC{0.28}  & \comvc   & \TRUE  & \MTC{0.30} & \MTC{0.25}  \\
    \mviic   & \FALSE & \MTC{0.27}  & \sae     & \FALSE & \MTC{0.56} & \mvae    & \FALSE   & \FALSE  & \MTC{0.38}  & \simvc   & \TRUE  & \MTC{0.31} & \MTC{0.16}  \\
    \comvc   & \FALSE & \MTC{0.29}  & \caekm   & \TRUE  & \MTC{0.59} & \dmsc    & \FALSE   & \FALSE  & \MTC{0.45}  & \sae     & \TRUE  & \MTC{0.33} & \MTC{0.06}  \\
    \mvae    & \FALSE & \MTC{0.43}  & \comvc   & \TRUE  & \MTC{0.63} & \mviic   & \TRUE    & \FALSE  & \MTC{0.98}  & \mvae    & \FALSE & \MTC{0.42} & \MTC{0.47}  \\
    \cae     & \FALSE & \MTC{0.65}  & \cae     & \TRUE  & \MTC{0.82} & \mimvc   & \TRUE    & \FALSE  & \MTC{1.15}  & \cae     & \TRUE  & \MTC{0.92} & \MTC{-0.13} \\
    \bottomrule
\end{tabular}
\egroup

%% file: fig/increasing_views/new_instances.tex
\begin{tikzpicture}

\definecolor{chocolate217952}{RGB}{217,95,2}
\definecolor{darkcyan27158119}{RGB}{27,158,119}
\definecolor{darkgray176}{RGB}{176,176,176}
\definecolor{lightgray204}{RGB}{204,204,204}

\begin{groupplot}[group style={group size=1 by 1}]
\nextgroupplot[
height=\figheight,
legend cell align={left},
legend style={
  fill opacity=0.8,
  draw opacity=1,
  text opacity=1,
  at={(0.03,0.03)},
  anchor=south west,
  draw=lightgray204
},
tick align=outside,
tick pos=left,
width=\figwidth,
x grid style={darkgray176},
xlabel={Number of views},
xmajorgrids,
xmin=1.8, xmax=6.2,
xtick style={color=black},
y grid style={darkgray176},
ymajorgrids,
ymin=0.2, ymax=0.5,
ytick style={color=black}
]
\path [draw=darkcyan27158119, fill=darkcyan27158119, opacity=0.2]
(axis cs:2,0.354057514141285)
--(axis cs:2,0.35014871411494)
--(axis cs:3,0.358439365623316)
--(axis cs:4,0.386306502300849)
--(axis cs:5,0.332119503090277)
--(axis cs:6,0.357462289521975)
--(axis cs:6,0.369810459425169)
--(axis cs:6,0.369810459425169)
--(axis cs:5,0.389725825717554)
--(axis cs:4,0.410165689509759)
--(axis cs:3,0.382401903869787)
--(axis cs:2,0.354057514141285)
--cycle;

\path [draw=chocolate217952, fill=chocolate217952, opacity=0.2]
(axis cs:2,0.39909291867472)
--(axis cs:2,0.341748350818383)
--(axis cs:3,0.34941393110675)
--(axis cs:4,0.34177315355269)
--(axis cs:5,0.335667346816696)
--(axis cs:6,0.386915402408217)
--(axis cs:6,0.41634104967537)
--(axis cs:6,0.41634104967537)
--(axis cs:5,0.360397781032883)
--(axis cs:4,0.362433074703535)
--(axis cs:3,0.383286178375062)
--(axis cs:2,0.39909291867472)
--cycle;

\addplot [semithick, darkcyan27158119]
table {%
2 0.352103114128113
3 0.370420634746552
4 0.398236095905304
5 0.360922664403915
6 0.363636374473572
};
\addlegendentry{\cae}
\addplot [semithick, chocolate217952, dashed]
table {%
2 0.370420634746552
3 0.366350054740906
4 0.352103114128113
5 0.348032563924789
6 0.401628226041794
};
\addlegendentry{\sae}
\end{groupplot}

\end{tikzpicture}

%% file: tab/ablation/both_ssl.tex
\rowcolors{2}{gray!25}{white}
\begin{tabular}{lcccc}
\toprule
& \multicolumn{2}{c}{NoisyMNIST}& \multicolumn{2}{c}{Caltech7}\\
Model & {\tiny \ST[c]{w/o\\SV-SSL}} & {\tiny \ST[c]{w/\\SV-SSL}} & {\tiny \ST[c]{w/o\\SV-SSL}} & {\tiny \ST[c]{w/\\SV-SSL}} \\ \cmidrule(lr){1-1} \cmidrule(lr){2-3} \cmidrule(lr){4-5}
\dmsc    & \MTC{0.54} & \MTCDELTA{0.66}{0.12} & \MTC{0.35} & \MTCDELTA{0.50}{0.15} \\
\sae     & \MTC{1.00} & \MTCDELTA{1.00}{0.00} & \MTC{0.41} & \MTCDELTA{0.40}{0.00} \\
\saekm   & \MTC{0.67} & \MTCDELTA{0.74}{0.07} & \MTC{0.39} & \MTCDELTA{0.44}{0.05} \\
\cae     & \MTC{1.00} & \MTCDELTA{1.00}{0.00} & \MTC{0.38} & \MTCDELTA{0.36}{-0.02} \\
\caekm   & \MTC{0.56} & \MTCDELTA{1.00}{0.44} & \MTC{0.22} & \MTCDELTA{0.20}{-0.02} \\
\midrule
\rowcolor{white} Model & {\tiny \ST[c]{w/o\\MV-SSL}} & {\tiny \ST[c]{w/\\MV-SSL}} & {\tiny \ST[c]{w/o\\MV-SSL}} & {\tiny \ST[c]{w/\\MV-SSL}} \\ \cmidrule(lr){1-1} \cmidrule(lr){2-3} \cmidrule(lr){4-5}
\eamc    & \MTC{1.00} & \MTCDELTA{0.83}{-0.17} & \MTC{0.36} & \MTCDELTA{0.44}{0.08} \\
\mvae    & \MTC{0.52} & \MTCDELTA{0.98}{0.46} & \MTC{0.31} & \MTCDELTA{0.47}{0.15} \\
\comvc   & \MTC{1.00} & \MTCDELTA{1.00}{0.00} & \MTC{0.41} & \MTCDELTA{0.38}{-0.02} \\
\cae     & \MTC{1.00} & \MTCDELTA{1.00}{0.00} & \MTC{0.40} & \MTCDELTA{0.36}{-0.04} \\
\caekm   & \MTC{0.74} & \MTCDELTA{1.00}{0.26} & \MTC{0.44} & \MTCDELTA{0.20}{-0.24} \\
\mimvc   & \MTC{1.00} & \MTCDELTA{0.90}{-0.10} & \MTC{0.41} & \MTCDELTA{0.51}{0.10} \\
\mviic   & \MTC{0.52} & \MTCDELTA{0.52}{0.00} & \MTC{0.53} & \MTCDELTA{0.53}{0.00} \\
\bottomrule
\end{tabular}

%% file: inputs/conclusion.tex
We investigate the role of self-supervised learning (SSL) in deep MVC.
Due to its recent success, we focus particularly on contrastive alignment, and prove that it can be detrimental to the clustering performance, especially when the number of views becomes large.
To properly evaluate models and components, we develop \fwName~-- a new unified framework for deep MVC, including the majority of recent methods as instances.
By leveraging the new insight from our framework and theoretical findings, we develop \( 6 \) new \fwName instances with several promising forms of SSL, which perform remarkably well compared to previous methods.
We conduct a thorough experimental evaluation of our new instances, previous methods, and their \fwName components -- and find that SSL is a crucial component in state-of-the-art methods for deep MVC.
In line with our theoretical analysis, we observe that contrastive alignment worsens performance when the number of views becomes large.
Further, we find that performance of methods depends heavily on dataset characteristics, such as number of views, and class imbalance.
Developing methods that are robust towards changes in these properties can thus result in methods that perform well over a wide range of multi-view clustering problems.
To this end, we make the following recommendations for future work in deep MVC:

\textbf{Improving contrastive alignment or maximization of mutual information to handle both few and many views.}
    Addressing pitfalls of alignment to improve contrastive alignment-based methods on many views, is a promising direction for future research.
    Similarly, we believe that improving the methods based on maximization of mutual information on few views, will result in better models.

\textbf{Developing end-to-end trainable clustering modules that are not biased towards balanced clusters.}
    The performance of the DDC clustering module illustrates the potential of end-to-end trainable clustering modules, which are capable of adapting the representations to produce better clusterings.
    Mitigating the bias towards balanced clusters thus has the potential to produce models that perform well, both on balanced and imbalanced datasets.

\textbf{Proper evaluation and open-source implementations.}
    Finally, we emphasize the importance of evaluating new methods on a representative collection of datasets, \eg many views and few views, paired, imbalanced, \etc.
    Also, in the reproducibility study (see supplementary), we find that original results can be difficult to reproduce.
    We therefore encourage others to use the open-source implementation of \fwName, as open code and datasets, and consistent evaluation protocols, are crucial to properly evaluate models and facilitate further development of new methods and components.

%% file: supp-inputs/intro.tex
Here, we provide the proofs for Propositions~\ref{prop:conditionalProbMinimum} and~\ref{prop:expectationMinimum};
additional details on the proposed new instances of \fwName;
the datasets used for evaluation;
the hyperparameters used by baselines and new instances;
and the computation of metrics and uncertainties used in our evaluation protocol.
We also include the full list of recent methods and their \fwName components, the complete table of results from the experimental evaluation.
In addition, we include additional experiments and analyses of reproducibility, hyperparameters, and the Fusion and CM components.
Finally, we reflect on possible negative societal impacts of our work.

Our implementation of the \fwName framework, as well as the datasets and the evaluation protocol used in our experiments, is available at \githubLink.
See \suppdir{README.md} in the repository for more details about the implementation, and how to reproduce our results.

%% file: supp-inputs/related-work.tex
\begin{table*}
    \centering
    \caption{Full overview of methods from previous work and their \fwName components.}
    \label{tab:previousMethodsFull}
    \tableFontSize
    \PreviousMethodsTable
    \MethodsAbbreviations
\end{table*}

The full list of recent methods and their \fwName components is given in Table~\ref{tab:previousMethodsFull}.
We observe that all but one model includes at least one form of SSL, but the type of SSL, and also fusion and CM, vary significantly for the different models.
This illustrates the importance of the SSL components in deep MVC, as well as the need for a unified framework with a consistent evaluation protocol, in order to properly compare and evaluate methods.

%% file: supp-inputs/alignment-deep-mvc.tex
\input{inputs/propositions.tex}
\customparagraph{Proof of propositions}
    \setcounter{proposition}{1}
    \begin{proof}[Proof of Proposition~\ref{prop:conditionalProbMinimum}]
        Let \( M_V = \min\limits_{v = 1, \dots, V} \{ k_v \} \), then we need to prove that
        \begin{align}
            \Prob(M_{V+1} \le M_V \mid k_1, \dots, k_V) = 1.
        \end{align}
        Due to the properties of the minimum operator, we have
        \begin{align}
            \begin{cases}
                M_{V+1} = M_V, \quad&\text{if } k_{V+1} \ge M_V \\
                M_{V+1} < M_V, &\text{otherwise}
            \end{cases}.
        \end{align}
        Hence, \( M_{V+1} \le M_V \) regardless of the value of \( k_{V+1} \), which gives
        \begin{align}
            \Prob(M_{V+1} \le M_V \mid k_1, \dots, k_V) = 1.
        \end{align}
    \end{proof}

    \begin{proof}[Proof of Proposition~\ref{prop:expectationMinimum}]
        Let \( M_V = \min\limits_{v = 1, \dots, V} \{ k_v \} \), then
        \begin{align}
            F_{M_V}(x) &\coloneqq \Prob(M_V \le x) = 1 - \Prob(M_V > x) \\
            &= 1 - \Prob(k_1 > x \cap \dots \cap k_V > x) \\
            &= 1 - (1 - F_{k_v}(x))^V
        \end{align}
        where \( F_{k_v}(x) = \Prob(k_v \le x) \).

        Since \( M_V \) is a non-negative random variable, we have
        \begin{align}
            \E(M_V) = \sums{x=0}{\infty}(1 - F_{M_V}(x)) = \sums{x=0}{\infty}(1 - F_{k_v}(x))^V.
        \end{align}
        Hence
        \begin{align}
            \nonumber
            &\E(M_V) - \E(M_{V+1}) \\
            &= \sums{x=0}{\infty}(1 - F_{k_v}(x))^V - \sums{x=0}{\infty}(1 - F_{k_v}(x))^{V+1} \\
            &= \sums{x=0}{\infty}(1 - F_{k_v}(x))^V (1 - (1 - F_{k_v}(x))) \\
            &= \sums{x=0}{\infty}
                \underbrace{(1 - F_{k_v}(x))^V}_{\ge 0}
                \underbrace{F_{k_v}(x) }_{\ge 0} \ge 0
        \end{align}
        which is a sum of non-negative terms, since \( F_{k_v}(x) \in [0, 1] \) is a probability.
        This gives
        \begin{align}
            \E(M_{V+1}) \le \E(M_V)
        \end{align}
    \end{proof}

%% file: supp-inputs/new-instances.tex
In this section we provide additional details on loss functions, particularly the weighted sum fusion, and the DDC~\cite{kampffmeyerDeepDivergencebasedApproach2019} clustering module.
The loss functions used to train the new instances are on the form
\begin{align}
    \TOTLoss = \SVSSLWeight \SVSSLLoss + \MVSSLWeight \MVSSLLoss + \CMWeight \CMLoss
\end{align}
where \( \SVSSLLoss \), \( \MVSSLLoss \), and \( \CMLoss \) denote the losses from the SV-SSL, MV-SSL, and CM components, respectively.
Note that the losses \( \SVSSLLoss \) and \( \MVSSLLoss \) correspond to the losses in Section 5 of the main paper.
\( (\SVSSLWeight, \MVSSLWeight, \CMWeight) \) are optional weights for the respective losses, which are all set to \( 1 \) unless specified otherwise.

\customparagraph{Connection between InfoDDC and contrastive self-supervised learning}
    For two views \( u \neq v \in 1, \dots, V \), contrastive SSL can be regarded as variational maximization of the mutual information
    \begin{align}
        I(\vec z^{(v)}, \vec z^{(u)})
    \end{align}
    where \( \vec z^{(v)} \) and \( \vec z^{(u)} \) have \emph{multi-variate, continuous} distributions in \( \real^d \).

    In InfoDDC, we instead maximize mutual information between pairs of \emph{uni-variate, discrete} random variables
    \begin{align}
        I(c^{(v)}, c^{(u)})
    \end{align}
    where we assume that the distributions of \( c^{(v)} \) and \( c^{(u)} \) are given by the view-specific representations
    \begin{align}
        \Prob(c^{(w)} = i) = z^{(w)}_{[i]},\quad i = 1, \dots, d, \quad w \in \{u, v\}
    \end{align}
    where \( z^{(w)}_{[i]} \) denotes component \( i \) of the view-specific representation \( \vec z^{(w)} = f^{(w)}(\vec x^{(w)}) \).
    Hence, although InfoDDC might appear similar to CA-based methods, the maximization of mutual information is done for different pairs of random variables.

\customparagraph{Weighted sum fusion.}
    As~\cite{trostenReconsideringRepresentationAlignment2021}, we implement the weighted sum fusion as
    \begin{align}
        \vec z_i = \sums{v=1}{V} w^{(v)} \vec z^{(v)}_i,
    \end{align}
    where the weights \( w^{(1)}, \dots, w^{(V)} \) are non-negative and sum to \( 1 \).
    These constraints are implemented by keeping a vector of trainable, un-normalized weights, from which \( w^{(1)}, \dots, w^{(V)} \) can be computed by applying the softmax function.

 \customparagraph{DDC clustering module.}
    The DDC~\cite{kampffmeyerDeepDivergencebasedApproach2019} clustering module consists of two fully-connected layers.
    The first layer calculates the hidden representation \( \vec h_i \in \real^{D_{DDC}} \) from the fused representation \( \vec z_i \).
    The dimensionality of the hidden representation, \( D_{DDC} \) is a hyperparameter set to \( 100 \) for all models.
    The second layer computes the cluster membership vector \( \vec \alpha_i \in \real^k \) from the hidden representation.

    DDC's loss function consists of three terms
    \begin{align}
        \CMLoss_{\text{DDC}} = \cl L_{\text{DDC, 1}} + \cl L_{\text{DDC, 2}} + \cl L_{\text{DDC, 3}}.
    \end{align}
    The three terms encourage
    \begin{enumerate*}[label=(\roman*)]
        \item separable and compact clusters in the hidden space;
        \item orthogonal cluster membership vectors; and
        \item cluster membership vectors close to simplex corners,
    \end{enumerate*}
    respectively.

    The first term maximizes the pairwise Cauchy-Schwarz divergence~\cite{jenssenCauchySchwarzDivergence2006} between clusters (represented as probability densities) in the space of hidden representations
    \begin{align}
        &\cl L_{\text{DDC, 1}} =\\
        &\binom{k}{2}^{-1}~~\sums{a=1}{k-1}\sums{b=a}{k} \frac{
            \sums{i=1}{n}\sums{j=1}{n} \alpha_{ia} \kappa_{ij} \alpha_{jb}
        }{
            \sqrt{\sums{i=1}{n}\sums{j=1}{n} \alpha_{ia} \kappa_{ij} \alpha_{ja} \sums{i=1}{n}\sums{j=1}{n} \alpha_{ib} \kappa_{ij} \alpha_{jb} }
        }
    \end{align}
    where \( \kappa_{ij} = \exp\lrp{-\frac{||\vec h_i - \vec h_j||^2}{2 \sigma^2}} \) and \( \sigma \) is a hyperparameter.
    Following~\cite{kampffmeyerDeepDivergencebasedApproach2019}, we set \( \sigma \) to \( 15\% \) of the median pairwise difference between the hidden representations.

    The second term minimizes the pairwise inner product between cluster membership vectors
    \begin{align}
        \cl L_{\text{DDC, 2}} = \frac{2}{n(n-1)} \sums{i=1}{n-1}\sums{j=i+1}{n} \vec \alpha_i \vec \alpha_j\T.
    \end{align}

    The third term encourages cluster membership vectors to be close to the corners of the probability simplex in \( \real^k \)
    \begin{align}
        &\cl L_{\text{DDC, 3}} =\\
        &\binom{k}{2}^{-1}~~\sums{a=1}{k-1}\sums{b=a}{k} \frac{
            \sums{i=1}{n}\sums{j=1}{n} m_{ia} \kappa_{ij} m_{jb}
        }{
             \sqrt{\sums{i=1}{n}\sums{j=1}{n} m_{ia} \kappa_{ij} m_{ja} \sums{i=1}{n}\sums{j=1}{n} m_{ib} \kappa_{ij} m_{jb} }
        }
    \end{align}
    where \( m_{ia} = \exp(-||\vec \alpha_i - \vec e_a||^2) \), and \( \vec e_a \) is the \( a \)-th simplex corner.

%% file: supp-inputs/experiments.tex
\subsection{Datasets}
    \begin{table*}
        \centering
        \caption{Dataset details. \( n \) = number of instances, \( v \) = number of views, \( k \) = number of classes/clusters, \( n_\text{small} \) = number of instances in smallest class, \( n_\text{big} \) = number of instances in largest class, Dim.~= view dimensions.}
        \label{tab:datasets}
        \undoCaptionSep
        \tableFontSize

\input{tab/datasets.tex}

    \end{table*}

    Dataset details are listed in Table~\ref{tab:datasets}.
    The code repository includes pre-processed Caltech7 and Caltech20 datasets.
    The other datasets can be generated by following the instructions in \suppdir{README.md} (these could not be included in the archive due to limitations on space).

    \customparagraph{Caltech details.}
        We use the same features and subsets of the Caltech101~\cite{fei-feiLearningGenerativeVisual2007} dataset as~\cite{huangMultiviewSpectralClustering2019}.
        \begin{itemize}
            \item \textbf{Features:}  Gabor, Wavelet Moments, CENsus TRansform hISTogram (CENTRIST), Histogram of Oriented Gradients (HOG), GIST, and Local Binary Patterns (LBP).
            \item \textbf{Caltech7 classes:} Face, Motorbikes, Dolla-Bill, Garfield, Snoopy, Stop-Sign, Windsor-Chair.
            \item \textbf{Caltech20 classes:} Face, Leopards, Motorbikes, Binocular, Brain, Camera,Car-Side, Dolla-Bill, Ferry, Garfield, Hedgehog, Pagoda, Rhino, Snoopy, Stapler, Stop-Sign, Water-Lilly, WindsorChair, Wrench, Yin-yang.
        \end{itemize}

\subsection{Hyperparameters}
    \begin{table*}
        \centering
        \caption{Network architectures.}
        \label{tab:arch}
        \tableFontSize
        \AllArch
    \end{table*}
    \customparagraph{Network architectures.}
    The encoder and decoder architectures are listed in Table~\ref{tab:arch}.
    MLP encoders/decoders are used for Caltech7 and Caltech20 as these contain vector data.
    The other datasets contain images, so CNN encoders and decoders are used for them.

    \customparagraph{Other hyperparameters.}
        \begin{table*}
            \centering
            \caption{Hyperparameters used to train the models. \EAMCLearningRate}
            \label{tab:hyperparametersAll}
            \tableFontSize
            \hyperparametersAll
        \end{table*}
        Table~\ref{tab:hyperparametersAll} lists other hyperparameters used for the baselines and new instances.

\subsection{Computational resources}
    We run our experiments on a Kubernetes cluster, where jobs are allocated to nodes with
    Intel(R) Xeon(R) E5-2623 v4 or Intel(R) Xeon(R) Silver 4210 CPUs (\( 2 \) cores allocated per job);
    and Nvidia GeForce GTX 1080 Ti or Nvidia GeForce RTX 2080 Ti GPUs.
    Each job has \( 16 \) GB RAM available.

    With this setup, \( 5 \) training runs on NoisyMNIST, NoisyFashion, EdgeMNIST, and EdgeFashion take approximately \( 24 \) hours.
    Training times for the other datasets are approximately between \( 1 \) and \( 3 \) hours.

    The Dockerfile used to build our docker image can be found in the code repository.

\subsection{Evaluation protocol}
    \customparagraph{Metrics.}
        We measure performance using the accuracy
        \begin{align}
            \text{ACC} = \max\limits_{m \in \cl M} \frac{\sum_{i=1}^{n} \delta(m(\hat y_i) - y_i)}{n}
        \end{align}
        where \( \delta(\cdot) \) is the Kronecker-delta, \( \hat y_i \) is the predicted cluster of instance \( i \), and \( y_i \) is the ground truth label of instance \( i \).
        The maximum runs over \( \cl M \), which is the set of all bijective mappings from \( \{ 1, \dots, k \} \) to itself.

        We also compute the normalized mutual information
        \begin{align}
            \text{NMI} = \frac{MI(\hat{\vec y}, \vec y)}{ \frac{1}{2}(H(\hat{\vec y}) + H(\vec y))}
        \end{align}
        where \( \hat{\vec y} = [\hat y_1, \dots, \hat y_n] \), \( \vec y = [y_1, \dots, y_n] \), \( MI(\cdot, \cdot) \) and \( H(\cdot) \) denotes the mutual information and entropy, respectively.

    \customparagraph{Uncertainty estimation.}
        The uncertainty of our performance statistic can be estimated using bootstrapping.
        Suppose the \( R \) training runs result in the \( R \) tuples
        \begin{align}
            (L_1, M_1), \dots, (L_R, M_R)
        \end{align}
        where \( L_i \) is the final loss of run \( i \), and \( M_i \) is resulting performance metric for run \( i \).
        We then sample \( B \) bootstrap samples uniformly from the original results
        \begin{align}
            & (L^b_j, M^b_j) \sim \text{Uniform}\{ (L_1, M_1), \dots, (L_R, M_R) \},\\
            \nonumber
            & \qquad j = 1, \dots, R, \quad b = 1, \dots B.
        \end{align}
        The performance statistic for bootstrap sample \( b \) is then given by
        \begin{align}
            M_{\star}^b = M^b_{j^b_{\star}}, \quad j^b_{\star} = \arg\min\limits_{j=1, \dots, R} \{ L^b_j \}.
        \end{align}
        We then estimate the uncertainty of the performance statistic by computing the standard deviation of the bootstrap statistics \( M_\star^1, \dots M_\star^B \)
        \begin{align}
            \hat\sigma_{M_\star} = \sqrt{\frac{\sum_{b=1}^{B} (M_{\star}^b - \bar{M}_{\star})^2 }{B-1}}, \text{ where}\quad \bar{M}_{\star} = \frac{\sum_{b=1}^{B} M^\star_b }{B}.
        \end{align}

\subsection{Results}
    \customparagraph{Evaluation results.}
        The complete evaluation results are given in Table~\ref{tab:fullResults}.
        \begin{table*}
            \centering
            \caption{Clustering results. Standard deviations (obtained by bootstrapping) are shown in parentheses. \( ^\dagger \) = training ran out of memory, \( ^\ddagger \) = training resulted in NaN loss.}
            \label{tab:fullResults}
            \tableFontSize
            \setlength{\tabcolsep}{.8mm}
            \input{tab/benchmark/benchmark1.tex}
            \input{tab/benchmark/benchmark2.tex}
        \end{table*}

    \customparagraph{Ablation study -- Fusion and Clustering module.}
        \begin{table}
            \centering
            \caption{Accuracies from ablation studies with the Fusion and CM components.}
            \label{tab:ablationFusionCM}
            \setlength{\tabcolsep}{.9mm}
            \tableFontSize%
            \begin{subtable}[t]{\columnwidth}
                \centering
                \caption{Fusion}
                \input{tab/ablation/fusion.tex}
            \end{subtable}
            \begin{subtable}[t]{\columnwidth}
                \centering
                \caption{CM}
                \input{tab/ablation/cm.tex}
            \end{subtable}
        \end{table}
        Table~\ref{tab:ablationFusionCM} shows the results of ablation studies with the fusion and clustering module (CM) components.
        Since these components can not be completely removed, we instead replace more complicated components, with the simplest possible component.
        Thus, we replace weighted sum with concatenate for the fusion component, and DDC with \( k \)-means for the CM component.

        For the fusion component, we see that the weighted sum tends to improve over the concatenation.
        For the CM, we observe that the performance is better with DDC than with \( k \)-means on NoisyMNIST, but the improvement more varied on Caltech7.
        This is consistent with what we observed in the evaluation results in the main paper.

    \customparagraph{Reproducibility of original results.}
        \begin{table}
            \bgroup
            \centering
            \caption{Accuracies from our experiment vs. accuracies reported by the original authors.
            \( ^\dagger \) = method is originally evaluated on a slightly different dataset.}
            \label{tab:refResults}
            \tableFontSize%
            \setlength{\tabcolsep}{.9mm}
            \input{tab/ref_results/ref_results.tex}
            \egroup
        \end{table}
        Table~\ref{tab:refResults} compares the results of our re-implementation of the baselines, to the results reported by the original authors.
        The comparison shows large differences in performance for several methods, and the differences are particularly large for MvSCN and Multi-VAE.
        For MvSCN, we do not use the same autoencoder preprocessing of the data.
        We also had difficulties getting the Cholesky decomposition to converge during training.
        For MultiVAE, we note that NoisyMNIST and NoisyFashion are generated without noise in the original paper, possibly resulting in datasets that are simpler to cluster.
        We were however not able to determine the reason for the difference in performance on COIL-20.
    
        Additionally, all methods use different network architectures and evaluation protocols in the original publications, making it difficult to accurately compare performance between methods and their implementations.
        This illustrates the difficulty of reproducing and comparing results in deep MVC, highlighting the need for a unified framework with a consistent evaluation protocol and an open-source implementation.

    \customparagraph{Sensitivity to hyperparameters}
        Table~\ref{tab:hpar} shows the results of hyperparameter sweeps for the following hyperparameters:
        \begin{itemize}
            \item Weight of reconstruction loss (\( \SVSSLWeight \)).
            \item Weight of contrastive loss (\( \MVSSLWeight \)).
            \item Temperature in contrastive loss (\( \tau \)).
            \item Weight of entropy regularization (\( \lambda \)).
        \end{itemize}
        We emphasize that these results were \emph{not} used to tune hyperparameters for the new instances.
        Rather, they are included to investigate how robust these methods are towards changes in the hyperparameter configuration.
        The results show that the new instances are mostly insensitive to changes in their hyperparameters.
        We however observe two cases where the hyperparameter configurations can have significant impact on the model performance.
        First, \cae shows a drop in performance when the weight of the contrastive loss is set to high on Caltech7 (Table~\ref{tab:hparConweight}).
        This is consistent with our observations regarding contrastive alignment on datasets with many views.
        Second, \mimvc and \mviic performs worse when the entropy regularization weight is set too low, indicating that sufficient regularization is required for these models to perform well.

        \begin{table*}
            \centering
            \caption{Results (NMI) of hyperparameter sweeps for the new instances.}
            \label{tab:hpar}
            \tableFontSize
            \setlength{\tabcolsep}{1mm}
            \begin{subtable}{\textwidth}
                \centering
                \caption{Weight of reconstruction loss (\( \SVSSLWeight \)).}
                \label{tab:hparRec}
                \input{tab/hpar/rec}
            \end{subtable}
            \begin{subtable}{\textwidth}
                \centering
                \caption{Weight of contrastive loss (\( \MVSSLWeight \)).}
                \label{tab:hparConweight}
                \input{tab/hpar/conweight}
            \end{subtable}
            \begin{subtable}{\textwidth}
                \centering
                \caption{Temperature in the contrastive loss (\( \tau \)).}
                \label{tab:hparContau}
                \input{tab/hpar/contau}
            \end{subtable}
            \begin{subtable}{\textwidth}
                \centering
                \caption{Weight of the entropy regularization (\( \lambda \)).}
                \label{tab:hparLambda}
                \input{tab/hpar/lambda}
            \end{subtable}
        \end{table*}

%% file: tab/datasets.tex
\bgroup
\setlength{\tabcolsep}{1.3mm}
\begin{tabular}{lrrrrrcl}
    \toprule
    \cthead{Dataset}                                           & \cthead{\( n \)} & \cthead{\( v \)} & \cthead{\( k \)} & \cthead{\( n_\text{small} \)} & \cthead{\( n_\text{big} \)} & \cthead{Dim.\ }                    & \cthead{Licence} \\ \cmidrule(lr){1-1} \cmidrule(lr){2-8}
    NoisyMNIST~\cite{lecunGradientbasedLearningApplied1998}   & \( 70000 \)     & \( 2 \)         & \( 10 \)        & \( 6313 \)                   & \( 7877 \)                 & \( (28 \times 28)^{2} \)          & CC BY-SA 3.0 \\
    NoisyFashion~\cite{xiaoFashionMNISTNovelImage2017}        & \( 70000 \)     & \( 2 \)         & \( 10 \)        & \( 7000 \)                   & \( 7000 \)                 & \( (28 \times 28)^{2} \)          & MIT \\
    EdgeMNIST~\cite{lecunGradientbasedLearningApplied1998}    & \( 70000 \)     & \( 2 \)         & \( 10 \)        & \( 6313 \)                   & \( 7877 \)                 & \( (28 \times 28)^{2} \)          & CC BY-SA 3.0 \\
    EdgeFashion~\cite{xiaoFashionMNISTNovelImage2017}         & \( 70000 \)     & \( 2 \)         & \( 10 \)        & \( 7000 \)                   & \( 7000 \)                 & \( (28 \times 28)^{2} \)          & MIT \\
    COIL-20~\cite{neneColumbiaObjectImage1996}               & \( 480 \)       & \( 3 \)         & \( 20 \)        & \( 24 \)                     & \( 24 \)                   & \( (64 \times 64)^{3} \)          & None \\
    Caltech7~\cite{fei-feiLearningGenerativeVisual2007}       & \( 1474 \)      & \( 6 \)         & \( 7 \)         & \( 34 \)                     & \( 798 \)                  & \( 48, 40, 254, 1984, 512, 928 \) & CC BY 4.0 \\
    Caltech20~\cite{fei-feiLearningGenerativeVisual2007}      & \( 2386 \)      & \( 6 \)         & \( 20 \)        & \( 33 \)                     & \( 798 \)                  & \( 48, 40, 254, 1984, 512, 928 \) & CC BY 4.0 \\
    PatchedMNIST~\cite{lecunGradientbasedLearningApplied1998} & \( 21770 \)     & \( 12 \)        & \( 3 \)         & \( 6903 \)                   & \( 7877 \)                 & \( (28 \times 28)^{12} \)         & CC BY-SA 3.0 \\
    \bottomrule
\end{tabular}

\egroup

%% file: tab/benchmark/benchmark1.tex
\begin{tabular}{lcccccccc}
\toprule
& \multicolumn{2}{c}{NoisyMNIST}& \multicolumn{2}{c}{NoisyFashion}& \multicolumn{2}{c}{EdgeMNIST}& \multicolumn{2}{c}{EdgeFashion} \\
& \multicolumn{1}{c}{ACC} & \multicolumn{1}{c}{NMI}& \multicolumn{1}{c}{ACC} & \multicolumn{1}{c}{NMI}& \multicolumn{1}{c}{ACC} & \multicolumn{1}{c}{NMI}& \multicolumn{1}{c}{ACC} & \multicolumn{1}{c}{NMI} \\ \cmidrule(lr){2-3} \cmidrule(lr){4-5} \cmidrule(lr){6-7} \cmidrule(lr){8-9}
\dmsc    & \MTC{ 0.66 }~\STD{ 0.02 } & \MTC{ 0.67 }~\STD{ 0.01 } & \MTC{ 0.49 }~\STD{ 0.05 } & \MTC{ 0.48 }~\STD{ 0.03 } & \MTC{ 0.51 }~\STD{ 0.02 } & \MTC{ 0.47 }~\STD{ 0.02 } & \MTC{ 0.52 }~\STD{ 0.01 } & \MTC{ 0.47 }~\STD{ 0.00 } \\
\mvscn   & \MTC{ 0.15 }~\STD{ 0.00 } & \MTC{ 0.02 }~\STD{ 0.00 } & \MTC{ 0.14 }~\STD{ 0.00 } & \MTC{ 0.01 }~\STD{ 0.00 } & \MTC{ 0.14 }~\STD{ 0.00 } & \MTC{ 0.01 }~\STD{ 0.01 } & \MTC{ 0.12 }~\STD{ 0.00 } & \MTC{ 0.03 }~\STD{ 0.00 } \\
\eamc    & \MTC{ 0.83 }~\STD{ 0.04 } & \MTC{ 0.90 }~\STD{ 0.02 } & \MTC{ 0.61 }~\STD{ 0.02 } & \MTC{ 0.71 }~\STD{ 0.02 } & \MTC{ 0.76 }~\STD{ 0.05 } & \MTC{ 0.79 }~\STD{ 0.03 } & \MTC{ 0.51 }~\STD{ 0.03 } & \MTC{ 0.47 }~\STD{ 0.01 } \\
\simvc   & \MTC{ \BEST{1.00} }~\STD{ 0.02 } & \MTC{ \BEST{1.00} }~\STD{ 0.02 } & \MTC{ 0.52 }~\STD{ 0.02 } & \MTC{ 0.51 }~\STD{ 0.02 } & \MTC{ 0.89 }~\STD{ 0.06 } & \MTC{ 0.90 }~\STD{ 0.04 } & \MTC{ 0.61 }~\STD{ 0.01 } & \MTC{ 0.56 }~\STD{ 0.02 } \\
\comvc   & \MTC{ \BEST{1.00} }~\STD{ 0.00 } & \MTC{ \BEST{1.00} }~\STD{ 0.00 } & \MTC{ 0.67 }~\STD{ 0.03 } & \MTC{ 0.68 }~\STD{ 0.03 } & \MTC{ \BEST{0.97} }~\STD{ 0.08 } & \MTC{ \BEST{0.94} }~\STD{ 0.07 } & \MTC{ 0.56 }~\STD{ 0.03 } & \MTC{ 0.52 }~\STD{ 0.01 } \\
\mvae    & \MTC{ 0.98 }~\STD{ 0.05 } & \MTC{ 0.96 }~\STD{ 0.02 } & \MTC{ 0.62 }~\STD{ 0.02 } & \MTC{ 0.60 }~\STD{ 0.01 } & \MTC{ 0.85 }~\STD{ 0.01 } & \MTC{ 0.76 }~\STD{ 0.01 } & \MTC{ 0.58 }~\STD{ 0.01 } & \MTC{ \BEST{0.64} }~\STD{ 0.00 } \\
\saekm   & \MTC{ 0.74 }~\STD{ 0.03 } & \MTC{ 0.71 }~\STD{ 0.00 } & \MTC{ 0.58 }~\STD{ 0.02 } & \MTC{ 0.59 }~\STD{ 0.01 } & \MTC{ 0.60 }~\STD{ 0.00 } & \MTC{ 0.57 }~\STD{ 0.00 } & \MTC{ 0.54 }~\STD{ 0.00 } & \MTC{ 0.58 }~\STD{ 0.00 } \\
\sae     & \MTC{ \BEST{1.00} }~\STD{ 0.04 } & \MTC{ \BEST{1.00} }~\STD{ 0.03 } & \MTC{ 0.69 }~\STD{ 0.06 } & \MTC{ 0.65 }~\STD{ 0.05 } & \MTC{ 0.88 }~\STD{ 0.11 } & \MTC{ 0.88 }~\STD{ 0.09 } & \MTC{ 0.60 }~\STD{ 0.01 } & \MTC{ 0.58 }~\STD{ 0.01 } \\
\caekm   & \MTC{ \BEST{1.00} }~\STD{ 0.00 } & \MTC{ 0.99 }~\STD{ 0.00 } & \MTC{ 0.63 }~\STD{ 0.07 } & \MTC{ 0.73 }~\STD{ 0.03 } & \MTC{ 0.38 }~\STD{ 0.03 } & \MTC{ 0.31 }~\STD{ 0.02 } & \MTC{ 0.39 }~\STD{ 0.04 } & \MTC{ 0.34 }~\STD{ 0.02 } \\
\cae     & \MTC{ \BEST{1.00} }~\STD{ 0.00 } & \MTC{ 0.99 }~\STD{ 0.00 } & \MTC{ \BEST{0.80} }~\STD{ 0.02 } & \MTC{ \BEST{0.77} }~\STD{ 0.01 } & \MTC{ 0.89 }~\STD{ 0.10 } & \MTC{ 0.90 }~\STD{ 0.09 } & \MTC{ \BEST{0.67} }~\STD{ 0.09 } & \MTC{ 0.62 }~\STD{ 0.06 } \\
\mimvc   & \MTC{ 0.90 }~\STD{ 0.05 } & \MTC{ 0.92 }~\STD{ 0.04 } & \MTC{ 0.54 }~\STD{ 0.03 } & \MTC{ 0.52 }~\STD{ 0.04 } & \MTC{ 0.62 }~\STD{ 0.04 } & \MTC{ 0.52 }~\STD{ 0.06 } & \MTC{ 0.43 }~\STD{ 0.01 } & \MTC{ 0.43 }~\STD{ 0.03 } \\
\mviic   & \MTC{ 0.52 }~\STD{ 0.04 } & \MTC{ 0.79 }~\STD{ 0.02 } & \MTC{ 0.52 }~\STD{ 0.07 } & \MTC{ 0.74 }~\STD{ 0.02 } & \MTC{ 0.31 }~\STD{ 0.04 } & \MTC{ 0.21 }~\STD{ 0.05 } & \MTC{ 0.52 }~\STD{ 0.04 } & \MTC{ 0.59 }~\STD{ 0.04 } \\
\bottomrule
\end{tabular}

%% file: tab/benchmark/benchmark2.tex
\begin{tabular}{lcccccccc}
\toprule
& \multicolumn{2}{c}{COIL-20}& \multicolumn{2}{c}{Caltech7}& \multicolumn{2}{c}{Caltech20}& \multicolumn{2}{c}{PatchedMNIST} \\
& \multicolumn{1}{c}{ACC} & \multicolumn{1}{c}{NMI}& \multicolumn{1}{c}{ACC} & \multicolumn{1}{c}{NMI}& \multicolumn{1}{c}{ACC} & \multicolumn{1}{c}{NMI}& \multicolumn{1}{c}{ACC} & \multicolumn{1}{c}{NMI} \\ \cmidrule(lr){2-3} \cmidrule(lr){4-5} \cmidrule(lr){6-7} \cmidrule(lr){8-9}
\dmsc    & \MTC{ - }\( ^\dagger \) ~\STD{ - } & \MTC{ - }\( ^\dagger \) ~\STD{ - } & \MTC{ 0.50 }~\STD{ 0.03 } & \MTC{ 0.50 }~\STD{ 0.02 } & \MTC{ 0.35 }~\STD{ 0.01 } & \MTC{ 0.55 }~\STD{ 0.00 } & \MTC{ - }\( ^\dagger \) ~\STD{ - } & \MTC{ - }\( ^\dagger \) ~\STD{ - } \\
\mvscn   & \MTC{ 0.21 }~\STD{ 0.00 } & \MTC{ 0.23 }~\STD{ 0.01 } & \MTC{ 0.29 }~\STD{ 0.02 } & \MTC{ 0.02 }~\STD{ 0.00 } & \MTC{ 0.13 }~\STD{ 0.01 } & \MTC{ 0.09 }~\STD{ 0.01 } & \MTC{ - }\( ^\dagger \) ~\STD{ - } & \MTC{ - }\( ^\dagger \) ~\STD{ - } \\
\eamc    & \MTC{ 0.39 }~\STD{ 0.15 } & \MTC{ 0.52 }~\STD{ 0.22 } & \MTC{ 0.44 }~\STD{ 0.02 } & \MTC{ 0.23 }~\STD{ 0.03 } & \MTC{ 0.22 }~\STD{ 0.04 } & \MTC{ 0.23 }~\STD{ 0.02 } & \MTC{ - }\( ^\ddagger \) ~\STD{ - } & \MTC{ - }\( ^\ddagger \) ~\STD{ - } \\
\simvc   & \MTC{ \BEST{0.90} }~\STD{ 0.04 } & \MTC{ \BEST{0.96} }~\STD{ 0.02 } & \MTC{ 0.41 }~\STD{ 0.02 } & \MTC{ 0.51 }~\STD{ 0.09 } & \MTC{ 0.34 }~\STD{ 0.02 } & \MTC{ 0.52 }~\STD{ 0.01 } & \MTC{ 0.84 }~\STD{ 0.04 } & \MTC{ 0.64 }~\STD{ 0.11 } \\
\comvc   & \MTC{ 0.87 }~\STD{ 0.03 } & \MTC{ \BEST{0.96} }~\STD{ 0.02 } & \MTC{ 0.38 }~\STD{ 0.01 } & \MTC{ 0.55 }~\STD{ 0.02 } & \MTC{ 0.34 }~\STD{ 0.01 } & \MTC{ 0.59 }~\STD{ 0.02 } & \MTC{ 0.73 }~\STD{ 0.12 } & \MTC{ 0.57 }~\STD{ 0.19 } \\
\mvae    & \MTC{ 0.74 }~\STD{ 0.02 } & \MTC{ 0.84 }~\STD{ 0.01 } & \MTC{ 0.47 }~\STD{ 0.02 } & \MTC{ 0.47 }~\STD{ 0.01 } & \MTC{ 0.40 }~\STD{ 0.01 } & \MTC{ 0.57 }~\STD{ 0.01 } & \MTC{ 0.94 }~\STD{ 0.00 } & \MTC{ 0.77 }~\STD{ 0.00 } \\
\saekm   & \MTC{ 0.88 }~\STD{ 0.04 } & \MTC{ 0.92 }~\STD{ 0.01 } & \MTC{ 0.44 }~\STD{ 0.03 } & \MTC{ 0.52 }~\STD{ 0.01 } & \MTC{ 0.45 }~\STD{ 0.02 } & \MTC{ 0.57 }~\STD{ 0.01 } & \MTC{ 0.87 }~\STD{ 0.00 } & \MTC{ 0.68 }~\STD{ 0.01 } \\
\sae     & \MTC{ 0.80 }~\STD{ 0.04 } & \MTC{ 0.93 }~\STD{ 0.02 } & \MTC{ 0.40 }~\STD{ 0.01 } & \MTC{ 0.54 }~\STD{ 0.07 } & \MTC{ 0.34 }~\STD{ 0.01 } & \MTC{ 0.44 }~\STD{ 0.03 } & \MTC{ 0.77 }~\STD{ 0.10 } & \MTC{ 0.59 }~\STD{ 0.17 } \\
\caekm   & \MTC{ 0.84 }~\STD{ 0.04 } & \MTC{ 0.94 }~\STD{ 0.02 } & \MTC{ 0.20 }~\STD{ 0.01 } & \MTC{ 0.05 }~\STD{ 0.00 } & \MTC{ 0.22 }~\STD{ 0.02 } & \MTC{ 0.27 }~\STD{ 0.02 } & \MTC{ 0.96 }~\STD{ 0.00 } & \MTC{ 0.85 }~\STD{ 0.00 } \\
\cae     & \MTC{ 0.87 }~\STD{ 0.01 } & \MTC{ \BEST{0.96} }~\STD{ 0.00 } & \MTC{ 0.36 }~\STD{ 0.01 } & \MTC{ 0.43 }~\STD{ 0.03 } & \MTC{ 0.31 }~\STD{ 0.02 } & \MTC{ 0.51 }~\STD{ 0.02 } & \MTC{ \BEST{0.99} }~\STD{ 0.00 } & \MTC{ \BEST{0.97} }~\STD{ 0.00 } \\
\mimvc   & \MTC{ 0.25 }~\STD{ 0.04 } & \MTC{ 0.54 }~\STD{ 0.03 } & \MTC{ 0.51 }~\STD{ 0.01 } & \MTC{ 0.60 }~\STD{ 0.04 } & \MTC{ \BEST{0.58} }~\STD{ 0.07 } & \MTC{ \BEST{0.63} }~\STD{ 0.03 } & \MTC{ \BEST{0.99} }~\STD{ 0.00 } & \MTC{ 0.96 }~\STD{ 0.00 } \\
\mviic   & \MTC{ 0.83 }~\STD{ 0.05 } & \MTC{ 0.94 }~\STD{ 0.02 } & \MTC{ \BEST{0.53} }~\STD{ 0.00 } & \MTC{ \BEST{0.63} }~\STD{ 0.04 } & \MTC{ 0.49 }~\STD{ 0.01 } & \MTC{ 0.61 }~\STD{ 0.01 } & \MTC{ 0.97 }~\STD{ 0.00 } & \MTC{ 0.90 }~\STD{ 0.01 } \\
\bottomrule
\end{tabular}

%% file: tab/ablation/fusion.tex
\rowcolors{2}{gray!25}{white}
\begin{tabular}{lcccc}
\toprule
& \multicolumn{2}{c}{NoisyMNIST}& \multicolumn{2}{c}{Caltech7}\\
Model & {\tiny Concat.} & {\tiny Weighted} & {\tiny Concat.} & {\tiny Weighted} \\ \cmidrule(lr){1-1} \cmidrule(lr){2-3} \cmidrule(lr){4-5}
\simvc   & \MTC{1.00} & \MTCDELTA{1.00}{0.00} & \MTC{0.36} & \MTCDELTA{0.41}{0.04} \\
\comvc   & \MTC{1.00} & \MTCDELTA{1.00}{0.00} & \MTC{0.42} & \MTCDELTA{0.38}{-0.04} \\
\sae     & \MTC{1.00} & \MTCDELTA{1.00}{0.00} & \MTC{0.36} & \MTCDELTA{0.40}{0.04} \\
\cae     & \MTC{1.00} & \MTCDELTA{1.00}{0.00} & \MTC{0.39} & \MTCDELTA{0.36}{-0.03} \\
\mimvc   & \MTC{0.93} & \MTCDELTA{0.90}{-0.03} & \MTC{0.36} & \MTCDELTA{0.51}{0.15} \\
\bottomrule
\end{tabular}

%% file: tab/ablation/cm.tex
\rowcolors{2}{gray!25}{white}
\begin{tabular}{lcccc}
\toprule
& \multicolumn{2}{c}{NoisyMNIST}& \multicolumn{2}{c}{Caltech7}\\
Model & {\tiny \( k \)-means } & {\tiny DDC} & {\tiny \( k \)-means } & {\tiny DDC} \\ \cmidrule(lr){1-1} \cmidrule(lr){2-3} \cmidrule(lr){4-5}
\simvc   & \MTC{0.67} & \MTCDELTA{1.00}{0.33} & \MTC{0.39} & \MTCDELTA{0.41}{0.01} \\
\comvc   & \MTC{0.56} & \MTCDELTA{1.00}{0.44} & \MTC{0.22} & \MTCDELTA{0.38}{0.16} \\
\sae     & \MTC{0.74} & \MTCDELTA{1.00}{0.26} & \MTC{0.44} & \MTCDELTA{0.40}{-0.04} \\
\cae     & \MTC{1.00} & \MTCDELTA{1.00}{0.00} & \MTC{0.20} & \MTCDELTA{0.36}{0.16} \\
\mimvc   & \MTC{0.14} & \MTCDELTA{0.90}{0.76} & \MTC{0.59} & \MTCDELTA{0.51}{-0.08} \\
\bottomrule
\end{tabular}

%% file: tab/ref_results/ref_results.tex
\begin{tabular}{llcc}
\toprule
Model                                      & Dataset                 & Orig.\     & Ours \\ \midrule
\rowcolor{gray!20}                         & N-MNIST                 & \MTC{0.99} & \MTCDELTA{0.15}{-0.84} \\
\rowcolor{gray!20}\multirow{-2}{*}{\mvscn} & Caltech20               & \MTC{0.59} & \MTCDELTA{0.13}{-0.46} \\
\multirow{-1}{*}{\eamc}                    & E-MNIST                 & \MTC{0.67} & \MTCDELTA{0.76}{0.09} \\
\rowcolor{gray!20}                         & E-MNIST                 & \MTC{0.86} & \MTCDELTA{0.89}{0.03} \\
\rowcolor{gray!20}                         & E-Fashion               & \MTC{0.57} & \MTCDELTA{0.61}{0.04} \\
\rowcolor{gray!20}\multirow{-3}{*}{\simvc} & COIL-20                 & \MTC{0.78} & \MTCDELTA{0.90}{0.12} \\
                                           & E-MNIST                 & \MTC{0.96} & \MTCDELTA{0.97}{0.01} \\
                                           & E-Fashion               & \MTC{0.60} & \MTCDELTA{0.56}{-0.04} \\
\multirow{-3}{*}{\comvc}                   & COIL-20                 & \MTC{0.89} & \MTCDELTA{0.87}{-0.02} \\
\rowcolor{gray!20}                         & N-MNIST\( ^\dagger \)   & \MTC{1.00} & \MTCDELTA{0.98}{-0.02} \\
\rowcolor{gray!20}                         & N-Fashion\( ^\dagger \) & \MTC{0.91} & \MTCDELTA{0.62}{-0.29} \\
\rowcolor{gray!20}\multirow{-3}{*}{\mvae}  & COIL-20                 & \MTC{0.98} & \MTCDELTA{0.74}{-0.24} \\
\bottomrule
\end{tabular}

%% file: tab/hpar/rec.tex
\begin{tabular}{lcccccccc}
\toprule
& \multicolumn{4}{c}{NoisyMNIST}& \multicolumn{4}{c}{Caltech7}\\
Rec.~weight & 0.01& 0.1& 1.0& 10.0& 0.01& 0.1& 1.0& 10.0\\ \cmidrule(lr){1-1} \cmidrule(lr){2-5} \cmidrule(lr){6-9}
\sae     & \MTC{ 1.00 }~\STD{ 0.03 } & \MTC{ 0.94 }~\STD{ 0.03 } & \MTC{ 1.00 }~\STD{ 0.03 } & \MTC{ 0.94 }~\STD{ 0.01 } & \MTC{ 0.41 }~\STD{ 0.01 } & \MTC{ 0.41 }~\STD{ 0.03 } & \MTC{ 0.44 }~\STD{ 0.02 } & \MTC{ 0.45 }~\STD{ 0.02 } \\
\cae     & \MTC{ 0.99 }~\STD{ 0.00 } & \MTC{ 0.99 }~\STD{ 0.00 } & \MTC{ 0.99 }~\STD{ 0.00 } & \MTC{ 0.99 }~\STD{ 0.00 } & \MTC{ 0.40 }~\STD{ 0.05 } & \MTC{ 0.33 }~\STD{ 0.02 } & \MTC{ 0.34 }~\STD{ 0.03 } & \MTC{ 0.49 }~\STD{ 0.04 } \\
\caekm   & \MTC{ 0.74 }~\STD{ 0.02 } & \MTC{ 0.70 }~\STD{ 0.04 } & \MTC{ 0.74 }~\STD{ 0.03 } & \MTC{ 0.93 }~\STD{ 0.02 } & \MTC{ 0.07 }~\STD{ 0.01 } & \MTC{ 0.05 }~\STD{ 0.00 } & \MTC{ 0.04 }~\STD{ 0.01 } & \MTC{ 0.04 }~\STD{ 0.02 } \\
\bottomrule
\end{tabular}

%% file: tab/hpar/conweight.tex
\begin{tabular}{lcccccccc}
\toprule
& \multicolumn{4}{c}{NoisyMNIST}& \multicolumn{4}{c}{Caltech7}\\
Con.~weight & 0.01& 0.1& 1.0& 10.0& 0.01& 0.1& 1.0& 10.0\\ \cmidrule(lr){1-1} \cmidrule(lr){2-5} \cmidrule(lr){6-9}
\cae     & \MTC{ 1.00 }~\STD{ 0.00 } & \MTC{ 0.99 }~\STD{ 0.00 } & \MTC{ 0.99 }~\STD{ 0.00 } & \MTC{ 0.99 }~\STD{ 0.00 } & \MTC{ 0.46 }~\STD{ 0.02 } & \MTC{ 0.40 }~\STD{ 0.05 } & \MTC{ 0.19 }~\STD{ 0.03 } & \MTC{ 0.09 }~\STD{ 0.01 } \\
\caekm   & \MTC{ 0.89 }~\STD{ 0.01 } & \MTC{ 0.73 }~\STD{ 0.03 } & \MTC{ 0.77 }~\STD{ 0.01 } & \MTC{ 0.67 }~\STD{ 0.02 } & \MTC{ 0.04 }~\STD{ 0.02 } & \MTC{ 0.04 }~\STD{ 0.01 } & \MTC{ 0.06 }~\STD{ 0.01 } & \MTC{ 0.05 }~\STD{ 0.00 } \\
\bottomrule
\end{tabular}

%% file: tab/hpar/contau.tex
\begin{tabular}{lcccccccc}
\toprule
& \multicolumn{4}{c}{NoisyMNIST}& \multicolumn{4}{c}{Caltech7}\\
\( \tau \) & 0.01& 0.07& 0.1& 1.0& 0.01& 0.07& 0.1& 1.0\\ \cmidrule(lr){1-1} \cmidrule(lr){2-5} \cmidrule(lr){6-9}
\cae     & \MTC{ 0.99 }~\STD{ 0.00 } & \MTC{ 1.00 }~\STD{ 0.00 } & \MTC{ 0.99 }~\STD{ 0.00 } & \MTC{ 1.00 }~\STD{ 0.00 } & \MTC{ 0.31 }~\STD{ 0.03 } & \MTC{ 0.39 }~\STD{ 0.01 } & \MTC{ 0.35 }~\STD{ 0.02 } & \MTC{ 0.48 }~\STD{ 0.01 } \\
\caekm   & \MTC{ 0.99 }~\STD{ 0.00 } & \MTC{ 0.91 }~\STD{ 0.02 } & \MTC{ 0.74 }~\STD{ 0.02 } & \MTC{ 0.78 }~\STD{ 0.05 } & \MTC{ 0.34 }~\STD{ 0.01 } & \MTC{ 0.05 }~\STD{ 0.01 } & \MTC{ 0.06 }~\STD{ 0.01 } & \MTC{ 0.46 }~\STD{ 0.01 } \\
\bottomrule
\end{tabular}

%% file: tab/hpar/lambda.tex
\begin{tabular}{lcccccccc}
\toprule
& \multicolumn{4}{c}{NoisyMNIST}& \multicolumn{4}{c}{Caltech7}\\
\( \lambda \) & 0.5& 1.5& 5.0& 10.0& 0.5& 1.5& 5.0& 10.0\\ \cmidrule(lr){1-1} \cmidrule(lr){2-5} \cmidrule(lr){6-9}
\mviic   & \MTC{ 0.03 }~\STD{ 0.01 } & \MTC{ 0.81 }~\STD{ 0.01 } & \MTC{ 0.82 }~\STD{ 0.00 } & \MTC{ 0.82 }~\STD{ 0.00 } & \MTC{ 0.04 }~\STD{ 0.01 } & \MTC{ 0.64 }~\STD{ 0.04 } & \MTC{ 0.60 }~\STD{ 0.01 } & \MTC{ 0.52 }~\STD{ 0.01 } \\
\mimvc   & \MTC{ 0.21 }~\STD{ 0.02 } & \MTC{ 0.37 }~\STD{ 0.02 } & \MTC{ 0.84 }~\STD{ 0.04 } & \MTC{ 0.94 }~\STD{ 0.07 } & \MTC{ 0.60 }~\STD{ 0.06 } & \MTC{ 0.60 }~\STD{ 0.02 } & \MTC{ 0.57 }~\STD{ 0.01 } & \MTC{ 0.51 }~\STD{ 0.01 } \\
\bottomrule
\end{tabular}

%% file: supp-inputs/negative-societal-impact.tex
As is the case with most methodological research, our work can be applied to downstream applications with negative societal impact -- for instance by reflecting biases in the dataset the model was trained on.
We note that in unsupervised learning, it is particularly important to check what a model has learned, due to the lack of label supervision.
This is crucial if the models are used to make high-stakes decisions.